\documentclass{article}

\usepackage{mathrsfs}
\usepackage{amssymb}
\usepackage{amsmath}
\usepackage{amsthm}
\usepackage{dsfont}
\usepackage{mathtools}
\usepackage{ascmac}
\usepackage{color}
\usepackage{bm}
\usepackage{comment}

\newcommand{\R}{\mathbb{R}}
\newcommand{\N}{\mathbb{N}}

\newcommand{\F}{\mathcal{F}}
\newcommand{\D}{\mathcal{D}}

\newcommand{\argmin}{\mathop{\rm arg~min}\limits}

\numberwithin{equation}{section} 

\newtheorem {rem}{Remark}[section]

\newtheorem {thm}[rem]{Theorem}
\newtheorem {lem}[rem]{Lemma}
\newtheorem {defi}[rem]{Definition}
\newtheorem {prop}[rem]{Proposition}

\newtheorem {assu}[rem]{Assumption}

\usepackage[top=4cm, bottom=4cm, left=3cm, right=3cm]{geometry}

\begin{document}
	
\title{Adam-like Algorithm with Smooth Clipping Attains Global Minima: Analysis Based on Ergodicity of Functional SDEs}
\author{Keisuke Suzuki}

\maketitle	

\begin{abstract}%
	In this paper, we prove that an Adam-type algorithm with smooth clipping approaches the global minimizer of the regularized non-convex loss function.
	Adding smooth clipping and taking the state space as the set of all trajectories, we can apply the ergodic theory of Markov semigroups for this algorithm and investigate its asymptotic behavior. 
	The ergodic theory we establish in this paper reduces the problem of evaluating the convergence, generalization error and discretization error of this algorithm to the problem of evaluating the difference between two functional stochastic differential equations (SDEs) with different drift coefficients. 
	As a result of our analysis, we have shown that this algorithm minimizes the the regularized non-convex loss function with errors of the form $n^{-1/2}$, $\eta^{1/4}$, $\beta^{-1} \log (\beta + 1)$ and $e^{- c t}$. 
	Here, $c$ is a constant and $n$, $\eta$, $\beta$ and $t$ denote the size of the training dataset, learning rate, inverse temperature and time, respectively.%
\end{abstract}

\section{Introduction}

As machine learning models trained with complex loss functions have become to play important roles, it becomes more and more important to guarantee the performance of non-convex optimization algorithms theoretically. 
For stochastic gradient Langevin dynamics (SGLD), which is defined by 
\begin{align}
	\label{EQ_SGLD_Disc}
	x_{k+1} 
	= x_k - \eta \nabla F(x_k) + \sqrt{2 \eta / \beta} \epsilon_k, 
\end{align}
there exist many studies investigating its theoretical performance \cite{Bertsekas, Ge, Ge2, Qian, Ragi, Liang, Abbasi, Mackey2, Mou, Xu, Kumar, Zhang, Majka, Kavis, Taiji, Taiji2, Moulines, Cucumber}. 
Here, $F : \R^d \to \R$ is an objective function, $\epsilon_k$ is a $d$-dimensional standard normal distributed noise, and $\eta$ and $\beta$ are the learning rate and inverse temperature, respectively. 
Assuming the dissipativity of $F$ and making use of the existence of a Lyapunov function for this chain, a number of the studies have shown that SGLD with appropriate hyperparameter settings can attain the global minima of the objective function even if it is non-convex. 
For example, according to \cite{Cucumber}, we have
\begin{align}
	\label{Main_Previous_Cucumber}
	E[F(x_k)] - \min_{w \in \R^d} F(w)
	\leq C_1 \left( \sqrt{\eta} + e^{-C_2 k \eta} \right) + C_3 \frac{\log(\beta +1)}{\beta}, 
\end{align}
where $C_1, C_2$ are constants independent of $\eta$ and $k$, and $C_3$ is a constant independent of $\eta$, $k$ and $\beta$. 
For the size $n$ of the training dataset, \cite{Cucumber} have also derived a generalization bound with the order $n^{-1}$ which is uniform at time $t = k \eta$, in the same way as the discretization error bound $\sqrt{\eta}$. 
Thus, we can determine hyperparameters in the order of $\beta$, $\eta$, and $k$ so that SGLD (\ref{EQ_SGLD_Disc}) minimizes the objective function $F$ within an arbitrary admissible error. 
In this way, these theoretical results both guarantee the performance of algorithms and suggest sufficient hyperparameter settings to learn models within an admissible error.
Therefore, a theoretical analysis of each algorithm is important in both the theoretical and applied fields. 

As extensions of SGLD, there are many gradient based optimization algorithms such as Adam \cite{kingma2015adam}, AdaBelief, Yogi \cite{zhuang2020adabelief}, AdaGrad and RMSProp \cite{chen2018convergence}. 
For example, Adam with $\varepsilon > 0$ and exponential decay rates $\beta_1, \beta_2 \in [0, 1)$ is described as follows. 
Here, $\odot$ denotes the Hadamard product and calculations such as divisions and the square root $\sqrt{}$ should be interpreted coordinate-wise. 
\begin{align}
	m_k &= \beta_1 m_{k-1} + (1-\beta_1) \nabla F(\theta_{k-1}), \label{EQ_Adam_Momentum1} \\
	v_k &= \beta_2 v_{k-1} + (1-\beta_2) [\nabla F(\theta_{k-1}) \odot \nabla F(\theta_{k-1})], \label{EQ_Adam_Momentum2} \\
	\theta_k &= \theta_{k-1} - \eta \frac{1}{\sqrt{\hat{v}_k + \varepsilon}} \odot \hat{m}_k, \label{EQ_Adam_Disc}
\end{align}
where $\hat{m}_k = m_k / (1 - \beta_1^k)$ and $\hat{v}_k = v_k / (1 - \beta_2^k)$. 
Adam can be considered as an accelerated version of SGLD making use of past information, and is widely used since it was proposed \cite{miyato2016adversarial, chambon2018deep, hamm2019deep, loey2021fighting}. 

Due to the widespread use of Adam, many researchers have investigated its theoretical performance \cite{zaheer2018adaptive, bock2019proof, reddi2019convergence, tang20211, zou2019sufficient, iiduka2022theoretical, zhang2022adam}. 
For example, \cite{zaheer2018adaptive} and \cite{zhang2022adam} have shown that the output $\theta_k$ of Adam approaches points where the gradient $\nabla F$ of $F$ vanishes. 
\cite{bock2019proof} has shown that Adam converges to the global minimizer of $F$ locally with an exponential rate of convergence even if $F$ is non-convex. 
On the basis of the importance of comparing SGLD with its continuation in the theoretical analysis, \cite{malladi2022sdes} has also derived the form of a stochastic differential equation (SDE), which appears as the continuous limit of Adam. 

However, the results of these previous studies alone are insufficient for the theoretical evaluation of Adam. 
In fact, the results in \cite{zaheer2018adaptive} and \cite{zhang2022adam} guarantee that Adam attains global minimum only when the loss function is convex. 
Furthermore, the result of \cite{bock2019proof} does not guarantee the global convergence, and \cite{malladi2022sdes} evaluated the discretization error only and the analysis of Adam based on this continuation was not conducted in their paper. 
Similarly, the other previously mentioned works do not guarantee the global convergence of Adam to the minimizers of the non-convex objective function. 

The difficulty when we analyze Adam is derived from its non-Markov property. 
Specifically, $\theta_k$ defined by (\ref{EQ_Adam_Disc}) is non-Markov as a $\R^d$-valued chain, unlike SGLD $x_k$ of (\ref{EQ_SGLD_Disc}). 
Although we can regard Adam as Markov considering the pair $(m_k, v_k, \theta_k)$ as in \cite{malladi2022sdes}, in this case, we cannot discuss its ergodicity because of the absence of its Lyapunov functions. 
To derive global convergence results such as (\ref{Main_Previous_Cucumber}), almost all existing studies for SGLD utilize its Markov property and ergodicity on the basis of the existence of a Lypunov function. 
For this reason, similar techniques to existing works for SGLD cannot be used for the analysis of Adam. 
Therefore, we can conclude that there are no techniques that give sufficient theoretical results to an Adam-type algorithm, and as a consequence, it is not yet known whether the Adam-type algorithm converges to the minimizer of the objective function globally.

In this paper, we generalize the result on the ergodicity of functional SDEs \cite{bao2020ergodicity}, and as its application, prove that the output of an Adam-type algorithm converges globally to the minimizer of the objective function. 
In a previous work \cite{Cucumber}, on the basis of the theory of Markov semigroups and its ergodicity, it has been shown that the analysis of SGLD reduced to that of the difference between two SDEs with different drift coefficients. 
Although an Adam-type algorithm is non-Markov as a $\R^d$-valued chain and does not have a Lyapunov function as a $\R^{3 d}$-valued chain, by taking the state space as the set of all trajectories as in \cite{bao2020ergodicity}, we can regard this chain as Markov and find its Lyapunov function. 
Therefore, making use of the result in \cite{bao2020ergodicity}, we can analyze the asymptotic behavior of Adam-type algorithms. 
In addition, extending the theory of this previous work to the case of the difference between two functional SDEs with different drift coefficients, we can also evaluate the generalization and discretization errors of Adam-type algorithms. 

This paper is organized as follows. 
In Section \ref{SEC_Main_Result}, we state our main result on the theoretical analysis of an Adam-type algorithm. 
The proof strategy of this result, which is based on the ergodicity of functional SDEs, is explained in Section \ref{SEC_Proof_Strategy}. 
Finally, in Section \ref{SEC_Future_Works}, we describe the problem of our results and our future works. 

\section{Main Result}
\label{SEC_Main_Result}

To state our main result, we prepare the following notations. 
Let $\mathcal{Z}$ be the set of all data points and $\ell(w; z)$ be the loss for a parameter $w \in \R^d$ on a data point $z \in \mathcal{Z}$. 
For independent and identically distributed samples $(z_1, \dots, z_n)$ generated from an unknown distribution $\D$ over $\mathcal{Z}$, we define the empirical loss $L_n$ and expected loss $L$ by $L_n(w) = \frac{1}{n} \sum_{i=1}^n \ell(w; z_i)$ and $L(w) = E_{z \sim \D}[\ell(w; z)]$, respectively. 
We impose the following assumptions on the loss function $\ell(\cdot; \cdot)$, which are standard in the theoretical analysis of SGLD (see \cite{Ragi, Zhang, Cucumber} for example). 

\begin{assu}
	\label{Assum_Loss_Function}
	Let $M > 0$ and $A > 0$. 
	For any $z \in \mathcal{Z}$, the function $\ell(\cdot; z) : \R^d \times \mathcal{Z} \to \R$ is $M$-smooth and satisfies 
	\begin{align*}
		\sup_{z \in \mathcal{Z}} |\ell(0; z)| < \infty,\quad 
		A \coloneqq \sup_{(x, z) \in \R^d \times \mathcal{Z}} \| \nabla \ell(x; z) \|_{\R^d} < \infty. 
	\end{align*}
	Here, a $C^1$-function $F : \R^d \to \R$ is said to be $M$-smooth if it satisfies 
	\begin{align*}
		\| \nabla F(x) - \nabla F(y) \|_{\R^d} 
		\leq M \| x - y \|_{\R^d},\quad x, y \in \R^d. 
	\end{align*}
\end{assu}

\begin{assu}
	\label{Assum_Regularization_Term}
	Let $M > 0$ and $m, b > 0$. 
	A function $R : \R^d \to \R$ is $M$-smooth and $(m, b)$-dissipative, and for any $(x, z) \in \R^d \times \mathcal{Z}$, $\nabla \ell(x; z) = 0$ holds whenever $\nabla R(x) \neq 0$. 
	Here, a $C^1$-function $F : \R^d \to \R$ is said to be $(m, b)$-dissipative if it satisfies 
	\begin{align*}
		\langle \nabla F(x), x \rangle_{\R^d} 
		&\geq m \| x \|_{\R^d}^2 - b,\quad x \in \R^d. 
	\end{align*}
\end{assu}

\begin{rem}
	\label{Rem_How_to_satisfy_Assum}
	Suppose that we want to find the minimizer of $L(w)$ on $\{ x \in \R^d \mid \| x \|_{\R^d} \leq K \}$ for some $K > 0$. 
	Then, we may change the value $\ell(x; z)$ to be larger for $\| x \|_{\R^d} > K$, and in particular, we may assume that $\ell(x; z)$ is a constant value when $\| x \|_{\R^d} > K+1$. 
	In this case, for any $\lambda > 0$, the function $R(x) = \lambda \left\{ (\| x \|_{\R^d}^2 + 1)^{1/2} - (K^2 + 1)^{1/2} \right\}^2 \mathds{1}_{\{ \| x \|_{\R^d} > K+1 \}}$ satisfies the conditions of Assumption \ref{Assum_Regularization_Term}. 
	Here, $\mathds{1}_\Gamma$ denotes the indicator function of a set $\Gamma$. 
	Hence, the aforementioned assumptions are satisfied in the $L^2$-regularization setting, for example. 
\end{rem}

Next, we describe the definition of our Adam-type algorithm. 
For the explanation how this algorithm relates to the original Adam, see Remark \ref{Rem_Our_Adam_Relate_Original} in the following. 
The following notations are derived from \cite{bao2020ergodicity}. 
Let $r > 0$. 
For a function $\xi : (-\infty, 0] \to \R^d$, we set $\| \xi \|_r = \sup_{-\infty < s \leq 0} (e^{r s} \| \xi(s) \|_{\R^d})$ and define $\mathcal{C}_r$ as the set of all $\xi$ with finite $\| \xi \|_r$. 
Furthermore, for $f : \R \to \R^d$ and $t \geq 0$, we define the shifted function $f_t : (-\infty, 0] \to \R^d$ by $f_t(s) = f(t+s)$. 
Finally, for $\eta > 0$ and a $M$-smooth function $F : \R^d \to \R$, we define the functionals $H^{(\eta)}_F, H_F : \mathcal{C}_r \to \R^d$ as follows. 
\begin{align}
	H_F^{(\eta)}(\xi) 
	&= - \displaystyle{\frac{(1 - e^{- c_1 \eta}) \sum_{j=-\infty}^0 e^{c_1 j \eta} \nabla F(\xi(j \eta))}{\sqrt{ \varepsilon + (1 - e^{- c_2 \eta}) \sum_{j=-\infty}^0 e^{c_2 j \eta} \| \nabla F(\xi(j \eta)) \|_{\R^d}^2}}}, \label{EQ_Adam_Drift_Coeff_Disc} \\
	H_F(\xi) 
	&= - \displaystyle{\frac{c_1 \int_{-\infty}^0 e^{c_1 s} \nabla F(\xi(s)) ds}{\sqrt{ \varepsilon + c_2 \int_{-\infty}^0 e^{c_2 s} \| \nabla F(\xi(s)) \|_{\R^d}^2 ds}}}, \label{EQ_Adam_Drift_Coeff_Conti} 
\end{align}
where $c_1, c_2 > 0$ and $\varepsilon > 0$ are fixed constants. 
Using $H^{(\eta)}_F$ and $H_F$, we define our Adam-type algorithm $X^{(\eta, \xi, F)}$ and its continuation $X^{(\xi, F)}$ as the solution of following functional SDEs with initial value $\xi \in \mathcal{C}_r$. 
\begin{align}
	d X^{(\eta, F)}(t) 
	&= H_F^{(\eta)}(X^{(\eta, F)}_{\lfloor t / \eta \rfloor \eta}) dt
	- \nabla R(X^{(\eta, F)}(\lfloor t / \eta \rfloor \eta)) dt + \sqrt{\frac{2}{\beta}} dW(t),  \label{EQ_Regularized_Adam_Disc} \\
	d X^{(F)}(t) 
	&= H_F(X^{(F)}_t) dt - \nabla R(X^{(F)}(t)) dt + \sqrt{\frac{2}{\beta}} dW(t), \label{EQ_Regularized_Adam_Conti} 
\end{align}
where $\lfloor \cdot \rfloor$ denotes the floor function and $W$ is a standard $d$-dimensional Brownian motion. 
Note that $X^{(\eta, \xi, F)}$ and $X^{(\xi, F)}$ define a $\mathcal{C}_r$-valued Markov chain and process, respectively although they are non-Markov as $\R^d$-valued ones. 
Therefore, by finding a Lypunov function for them, we can use the ergodic theory to evaluate their asymptotic behaviors. 

Under the aforementioned notations, our main result is described as follows. 
Here, for a vector of parameters $\alpha$, $f \leq O_\alpha(g)$ and $f \geq \Omega_\alpha(g)$ mean that there exists constants $C$ and $c$ depending only on $\alpha$ such that $f \leq C g$ and $f \geq c g$ hold, respectively. 

\begin{thm}
	\label{Thm_Main_Result}
	Suppose that $2 c_1 > c_2$, $\min \{ c_1, c_2\} > r$, $m / 3 > r$ and $\beta \geq 2 / m$ hold. 
	Then, under Assumptions \ref{Assum_Loss_Function} and \ref{Assum_Regularization_Term}, for $\alpha = (c_1, c_2, \varepsilon, \| \nabla R(0) \|_{\R^d}, A, m, b, M, r, d)$, there exist $t_0 = O_{\alpha, \beta}(1)$ and $c = \Omega_{\alpha, \beta}(1)$ such that 
	\begin{align*}
		&E[(L + \varepsilon^{1/2} R)(X^{(\eta, \xi, L_n)}(k \eta))] - \min_{w \in \R^d} (L + \varepsilon^{1/2} R)(w) \\
		&\quad\leq O_{\alpha, \beta} \left( (1 + \| \xi \|_r^{5/2}) (\eta^{1/4} + n^{-1/2} + e^{- c k \eta} (R(w^*) + 1)) \right) 
		+ O_\alpha\left( \frac{R(w^*) + \log (\beta + 1)}{\beta} \right)
	\end{align*}
	holds uniformly on $0 < \eta \leq 1$ and $k$ satisfying $k \eta \geq t_0$. 
	Here, $w^* = \argmin_{w \in \R^d} (L + \varepsilon^{1/2} R)(w)$. 
\end{thm}

Therefore, the Adam-type algorithm $X^{(\eta, \xi, L_n)}$ for the empirical loss $L_n$ approaches the global minimizer of the regularized expected loss $L + \varepsilon^{1/2} R$ with appropriate hyper parameters $\eta$, $\beta$ and the size $n$ of the training data. 
To the best of our knowledge, this is the first result that guarantees that Adam-type algorithms converge globally to the minimizer of the non-convex objective function. 

\begin{rem}
	\label{Rem_Our_Adam_Relate_Original}
	For Adam defined by (\ref{EQ_Adam_Momentum1}),  (\ref{EQ_Adam_Momentum2}) and (\ref{EQ_Adam_Disc}), by induction, we can show 
	\begin{align*}
		m_k 
		&= (1-\beta_1) \sum_{j=-(k-1)}^0 \beta_1^{-j} \nabla F(\theta_{k-1+j}), \\
		v_k 
		&= (1-\beta_2) \sum_{j=-(k-1)}^0 \beta_2^{-j} [\nabla F(\theta_{k-1+j}) \odot \nabla F(\theta_{k-1+j})]. 
	\end{align*}
	Therefore, omitting the normalization $m_k$ and $v_k$ to $\hat{m}_k$ and $\hat{v}_k$, the updating rule of Adam becomes
	\begin{align}
		\label{EQ_Biased_Adam}
		\theta_k 
		&= \theta_{k-1} 
		- \eta \frac{1}{\sqrt{\varepsilon + (1-\beta_2) \sum_{j=-k-1}^0 \beta_2^{-j} [\nabla F(\theta_{k-1+j}) \odot \nabla F(\theta_{k-1+j})]}} \notag \\
		&\quad\odot (1-\beta_1) \sum_{j=-(k-1)}^0 \beta_1^{-j} \nabla F(\theta_{k-1+j}). 
	\end{align}
	Assuming the relation $\beta_1 = e^{-c_1 \eta}$ and $\beta_2 = e^{-c_2 \eta}$ as in \cite{malladi2022sdes}, (\ref{EQ_Regularized_Adam_Disc}) with constant initial value $\xi \equiv \theta_0$ is obtained by (i) extending the ranges of summations in the numerator and denominator, (ii) adding the smooth clipping $- \nabla R$ and noise term $\sqrt{2 / \beta} dW$, and (iii) changing the adjustment rule of the learning rate to the same form for each coordinate, in (\ref{EQ_Biased_Adam}). 
	(i) and (ii) are needed for (\ref{EQ_Regularized_Adam_Disc}) and (\ref{EQ_Regularized_Adam_Conti}) to define stationary Markov semigroups, and (iii) is needed to evaluate the limiting distribution of (\ref{EQ_Regularized_Adam_Conti}), respectively. 
\end{rem}

\section{Proof Strategy of Theorem \ref{Thm_Main_Result}}
\label{SEC_Proof_Strategy}

First, we establish a technique to evaluate the difference between two functional SDEs with different drift coefficients. 
Let $R : \R^d \to \R$ be $M$-smooth and $(m, b)$-dissipative.
Furthermore, let functionals $\mathcal{H}, H : \mathcal{C}_r \to \R^d$ satisfy the boundedness $\| \mathcal{H}(\xi) \|_{\R^d}, \| H(\xi) \|_{\R^d} \leq K$ and the Lipschitz continuity $\| \mathcal{H}(\xi) - \mathcal{H}(\xi^\prime) \|_{\R^d}, \| H(\xi) - H(\xi^\prime) \|_{\R^d} \leq K \| \xi - \xi^\prime \|_r$ for some $K > 0$. 
Then, we define $\mathcal{X}^{(\xi)}$, $X^{(\eta, \xi)}$ and $X^{(\xi)}$ as the solutions of the following functional SDEs with initial value $\xi \in \mathcal{C}_r$. 
\begin{align}
	d \mathcal{X}(t) 
	&= \mathcal{H}(\mathcal{X}_t) dt - \nabla R(\mathcal{X}(t)) dt + \sqrt{\frac{2}{\beta}} dW(t), \label{EQ_FSDE_Conti} \\
	dX^{(\eta)}(t) 
	&= H(X^{(\eta)}_{\lfloor t / \eta \rfloor \eta}) dt - \nabla R(X^{(\eta)}(\lfloor t / \eta \rfloor \eta)) dt + \sqrt{\frac{2}{\beta}} dW(t), \label{EQ_FSDE_Disc} \\
	d X(t) 
	&= H(X_t) dt - \nabla R(X(t)) dt + \sqrt{\frac{2}{\beta}} dW(t). \notag
\end{align}
Under these conditions and notations, on the basis of the techniques established in \cite{bao2020ergodicity}, we can find a Lyapunov function (see \cite{hairer2011asymptotic} or \cite{bao2020ergodicity} for its definition) for these dynamics and investigate their asymptotic behaviors. 
As a consequence of such a discussion, the differences between $\mathcal{X}^{(\xi)}$ and $X^{(\eta, \xi)}$, or $\mathcal{X}^{(\xi)}$ and $X^{(\xi)}$, can be evaluated as follows. 

\begin{thm}
	\label{Thm_General_Bound_on_Difference}
	Suppose that $F : \R^d \to \R$ is $M$-smooth, and $\beta \geq 2 / m$ and $m / 3 > r$ hold. 	
	Then, for $\alpha = (\| \nabla F(0) \|_{\R^d}, \| \nabla R(0) \|_{\R^d}, K, m, b, M, r, d)$, $|E[F(\mathcal{X}^{(\xi^\prime)}(k \eta))] - E[F(X^{(\eta, \xi)}(k \eta))]|$ is bounded by
	\begin{align}
		&(1 + \| \xi \|_r^2 + \| \xi^\prime \|_r^2) \left\{ \| \xi - \xi^\prime \|_r^{1/2} e^{- c k \eta} \right. \notag \\
		&\left.\quad+ C 
		\left( \sup_{s \geq 0} E[\| H(X^{(\eta, \xi)}_s) - \mathcal{H}(X^{(\eta, \xi)}_s) \|_{\R^d}^2]^{1/4} 
		+ \eta^{1/4} (1 + \| \xi \|_r^{1/2}) \right) \right\} 
		\label{EQ_General_Bound_on_Difference}
	\end{align}
	uniformly on $0 < \eta \leq 1$ and $k$ satisfying $k \eta \geq t_0$, where $c = \Omega_{\alpha, \beta}(1)$ and $C, t_0 = O_{\alpha, \beta}(1)$. 
	Thus, as a consequence of the limit $\eta \to 0$, $|E[F(\mathcal{X}^{(\xi^\prime)}(t))] - E[F(X^{(\xi)}(t))]|$ is bounded by
	\begin{align}
		\label{EQ_General_Bound_on_Difference_Conti}
		(1 + \| \xi \|_r^2 + \| \xi^\prime \|_r^2) \left\{ \| \xi - \xi^\prime \|_r^{1/2} e^{- c t} 
		+ C \sup_{s \geq 0} E[\| H(X^{(\xi)}_s) - \mathcal{H}(X^{(\xi)}_s) \|_{\R^d}^2]^{1/4} \right\}  
	\end{align}
	uniformly on $k$ satisfying $k \eta \geq t_0$. 
\end{thm}

For the proof of Theorem \ref{Thm_General_Bound_on_Difference}, see Appendix \ref{SEC_Proof_Main}. 

In the rest of this section, we assume Assumptions \ref{Assum_Loss_Function} and \ref{Assum_Regularization_Term}. 
According to Theorem \ref{Thm_General_Bound_on_Difference}, to evaluate the difference between $X^{(\eta, \xi, L_n)}$ and $X^{(\xi, L_n)}$, we only have to evaluate the difference between $H^{(\eta)}_{L_n}$ and $H_{L_n}$. 
Similarly, finding an upper bound of the difference between $H_L$ and $H_{L_n}$ is sufficient to evaluate the difference between $X^{(\xi, L)}$ and $X^{(\xi, L_n)}$. 
Adopting this strategy, we can prove the following. 

\begin{prop}
	\label{Lem_StepSize_and_Gen_Bound_Adam}
	Under the same conditions and notations as Theorem \ref{Thm_Main_Result}, the following inequalities hold uniformly on $0 < \eta \leq 1$ and k satisfying $k \eta \geq t_0$. 
	\begin{align*}
		&\left| E[(L_n + \varepsilon^{1/2} R)(X^{(\eta, \xi, L_n)}(k \eta))] - E[(L_n + \varepsilon^{1/2} R)(X^{(\xi, L_n)}(k \eta))] \right| \\
		&\quad\leq O_{\alpha, \beta} \left( (1 + \| \xi \|_r^{5/2}) \eta^{1/4} \right), \\
		&\left| E[(L + \varepsilon^{1/2} R)(X^{(\xi, L_n)}(t))] - E[(L_n + \varepsilon^{1/2} R)(X^{(\xi, L_n)}(t))] \right| 
		\leq  O_{\alpha, \beta} \left( (1 + \| \xi \|_r^{5/2}) n^{-1/2} \right). 
	\end{align*}
\end{prop}

On the other hand, by applying (\ref{EQ_General_Bound_on_Difference_Conti}) to the case of $H = \mathcal{H} = H_{L_n}$, we obtain the exponential convergence of $X^{(\xi, L_n)}$ to its limiting distribution. 
Furthermore, using (3.4) in \cite{Ragi}, we can prove that this limiting distribution concentrates on the set of all minimizers of $L_n + \varepsilon^{1/2} R$. 
As a consequence of such a discussion, we can obtain the following. 

\begin{prop}
	\label{Lem_Gradient_Vanish_in_Adam}
	Under the same conditions and notations as Theorem \ref{Thm_Main_Result}, the following holds uniformly on $t \geq t_0$. 
	\begin{align*}
		&E[(L_n + \varepsilon^{1/2} R)(X^{(\xi, L_n)}(t))] - \min_{w \in \R^d} (L + \varepsilon^{1/2} R)(w) \\
		&\quad\leq O_{\alpha, \beta}(e^{- c t} (1 + \| \xi \|_r^{5/2}) \{ R(w^*) + 1 \}) 
		+ O_\alpha\left( \frac{R(w^*)}{\beta} + \frac{\log (\beta + 1)}{\beta} \right). 
	\end{align*}
\end{prop}

The proofs of Propositions \ref{Lem_StepSize_and_Gen_Bound_Adam} and \ref{Lem_Gradient_Vanish_in_Adam} are also given in Appendix  \ref{SEC_Proof_Main}. 

Putting these inequalities together into 
\begin{align*}
	&E[(L + \varepsilon^{1/2} R)(X^{(\eta, \xi, L_n)}(k \eta))] - \min_{w \in \R^d} (L + \varepsilon^{1/2} R)(w) \\
	&\quad\leq E[(L + \varepsilon^{1/2} R)(X^{(\eta, \xi, L_n)}(k \eta))] - E[(L + \varepsilon^{1/2} R)(X^{(\xi, L_n)}(k \eta))] \\
	&\qquad+ E[(L + \varepsilon^{1/2} R)(X^{(\xi, L_n)}(k \eta))] - E[(L_n + \varepsilon^{1/2} R)(X^{(\xi, L_n)}(k \eta))] \\
	&\qquad+ E[(L_n + \varepsilon^{1/2} R)(X^{(\xi, L_n)}(k \eta))] - \min_{w \in \R^d} (L + \varepsilon^{1/2} R)(w), 
\end{align*}
we obtain Theorem \ref{Thm_Main_Result}. 

\section{Future Works}
\label{SEC_Future_Works}

In this paper, we have shown that the Adam-type algorithm (\ref{EQ_Regularized_Adam_Disc}) converges to the minimizer of the regularized expected loss function globally with an exponential convergence rate. 
In addition, for the size $n$ of the training dataset and learning rate $\eta$, we also have evaluated the generalization and discretization errors of this algorithm uniformly at time by quantities of orders $n^{-1/2}$ and $\eta^{1/4}$, respectively. 
Since our technique is based on the general theory of the ergodicity of functional SDEs, it also can be applied to the analysis of other algorithms such as AdaBelief and AdaGrad. 

However, as shown in \cite{Cucumber}, for SGLD, these errors are bounded by quantities with higher orders $n^{-1}$ and $\eta^{1/2}$. 
The reason for this disagreement is attributed to the exponent $1/4$ in Theorem \ref{Thm_General_Bound_on_Difference}. 
Thus, to increase the orders of these error bounds for the Adam-type algorithm (\ref{EQ_Regularized_Adam_Disc}), we have to establish a sharper technique of evaluating the difference between two functional SDEs. 
In addition, in our algorithm (\ref{EQ_Regularized_Adam_Disc}), the adaptive adjustment of the learning rate is not coordinate-wise. 
This technical requirement is needed to evaluate the limiting distribution of (\ref{EQ_Regularized_Adam_Disc}), but is not natural and should be removed in the future. 

\appendix

\section{Results on Ergodicity of Markov Semigroups}
\label{SEC_Ergodicity_General}

In this section, we extend the result of the ergodicity \cite{hairer2011asymptotic} to the difference of two Markov semigroups (see Theorem \ref{Thm_Uniqueness_Invariant_Measure}). 
This general theory will be used in Appendix \ref{SEC_Ergodicity_FSDE} for the preparation of the proof of Theorem \ref{Thm_General_Bound_on_Difference}. 

Let $(S, \rho)$ be a Polish space and let $\mathcal{B}(S)$ be the $\sigma$-algebra generated from the topology determined by $\rho$. 
We denote the set of all probability measures on $(S, \mathcal{B}(S))$ by $\mathcal{P}(S)$.   

\begin{defi}
	\label{Def_Distance_Like_Func}
	A symmetric function $d : S \times S \to [0, \infty)$ is said to be distance-like if it is lower semicontinuous with respect to the product topology and satisfies $d(x, y) = 0 \Leftrightarrow x = y$. 
\end{defi}

\begin{defi}
	\label{Def_Wasserstein_Type_Distance}
	Let $d : S \times S \to [0, \infty)$ be distance-like. 
	For $\mu, \nu \in \mathcal{P}(S)$, we define the set of all couplings of them by 
	\begin{align*}
		\Pi(\mu, \nu) 
		\coloneqq \{ \pi \in \mathcal{P}(S \times S) \mid \pi(\cdot \times S) = \mu, \pi(S \times \cdot) = \nu \}  
	\end{align*}
	and the Wasserstine type difference of them by
	\begin{align*}
		\mathbb{W}_d(\mu, \nu)
		\coloneqq \inf_{\pi \in \Pi(\mu, \nu)} \int_{S \times S} d(x, y) \pi(dx dy). 
	\end{align*}
\end{defi}

\begin{defi}
	\label{Def_Markov_Kernel}
	$P: S \times \mathcal{B}(S) \to [0, 1]$ is said to be a Markov kernel on $S$ if it satisfies the following two conditions. 
	\begin{enumerate}
		\item For any $x \in S$, $P(x, \cdot)$ is a probability measure on $\mathcal{B}(S)$. 
		\item For any $A \in \mathcal{B}(S)$, $P(\cdot, A)$ is a $\mathcal{B}(S)$-measurable function. 
	\end{enumerate}
\end{defi}

\begin{defi}
	\label{Def_d_small}
	Let $d : S \times S \to [0, 1]$ be distance-like and let $P, Q: S \times \mathcal{B}(S) \to [0, 1]$ be Markov kernels on $S$. 
	For a constant $\varepsilon > 0$, a set $A \subset S$ is said to be $\varepsilon$-$d$-small for $(P, Q)$ if 
	\begin{align*}
		\mathbb{W}_d (P(x, \cdot), Q(y, \cdot)) 
		\leq 1 - \varepsilon
	\end{align*}
	holds for any $x, y \in A$. 
\end{defi}

\begin{defi}
	\label{Def_Alpha_Contracting}
	Let $d : S \times S \to [0, 1]$ and $P, Q: S \times \mathcal{B}(S) \to [0, 1]$ be the same as in Definition \ref{Def_d_small}. 
	For a constant $\theta \in (0, 1)$ and a function $\varDelta : S \to [0, \infty)$, $d$ is said to be a $(\theta, \varDelta)$-contracting for $(P, Q)$ if  
	\begin{align*}
		\mathbb{W}_d( P(x, \cdot), Q(y, \cdot)) 
		\leq \theta d(x, y) + \varDelta(x)
	\end{align*}
	holds for any $x, y \in S$ satisfying $d(x, y) < 1$. 
\end{defi}

\begin{defi}
	\label{Def_Markovian_Semigroup}
	Let $\mathcal{T}$ be $[0, \infty)$ or $\{ k \eta \}_{k=0}^\infty$. 
	$\{ P_t \}_{t \in \mathcal{T}}$ is said to be a Markov semigroup on $S$ if each $P_t$ is a Markov kernel on $S$ and $P_t P_s = P_{t+s}$ holds for any $s, t \in \mathcal{T}$. 
	A measurable function $V : S \to [0, \infty)$ is said to be a Lyapunov function for $\{ P_t \}_{t \in \mathcal{T}}$ if the following two conditions are satisfied. 
	\begin{enumerate}
		\item For any $t \in \mathcal{T}$ and $x \in S$, $P_t(x, V) \coloneqq \int_S V(y) P_t(x, dy) < \infty$ holds, 
		\item There exist constants $C_V, \gamma, K_V > 0$ such that  
		\begin{align}
			\label{EQ_Lyapunov_Condition}
			P_t(x, V) 
			\leq C_V e^{- \gamma t} V(x) + K_V
		\end{align}
		holds for any $t \in \mathcal{T}$ and $x \in S$. 
	\end{enumerate}
\end{defi}

The following is an extension of Theorem 4.8 in \cite{hairer2011asymptotic} to the difference of two Markov semigroups. 

\begin{thm}
	\label{Thm_Uniqueness_Invariant_Measure}
	Let $\{ P_t \}_{t \in \mathcal{T}}$ and $\{ Q_t \}_{t \in \mathcal{T}}$ be Markov semigroups on $S$ and let $V$ be a common Lyapunov function of them. 
	Suppose that the following conditions are satisfied for a constant $t_* > \log (8 C_V) / \gamma$ and a distance-like function $d : S \times S \to [0, 1]$. 
	\begin{enumerate}
		\item There exist some $\theta \in (0, 1)$ and $\varDelta : S \to [0, \infty)$ such that $d$ is a $(\theta, \varDelta)$-contracting for $(P_{t_*}, Q_{t_*})$. 
		\item There exists some $\varepsilon > 0$ such that the set $\{ x \in S \mid V(x) \leq 4 K_V \}$ is $\varepsilon$-$d$-small for $(P_{t_*}, Q_{t_*})$. 
	\end{enumerate}
	Finally, we define 
	\begin{align*}
		\tilde{d}(x, y) 
		= \sqrt{d(x, y) (1 + V(x) + V(y))}.  
	\end{align*}
	Then, for $\alpha = (\theta, \varepsilon, \gamma, K_V)$, we can take a natural number $k_0 = O_\alpha(1)$ and a constant $C = O_\alpha(1)$ so that
	\begin{align}
		\label{EQ_Markov_Kernel_Contraction}
		\mathbb{W}_{\tilde{d}}(P_{k_0 t_*} \mu, Q_{k_0 t_*} \nu)
		\leq \frac{1}{2} \mathbb{W}_{\tilde{d}}(\mu, \nu) + C \sum_{j=0}^{k_0-1} \sqrt{1 + (P_{j t_*} \mu)(V) + (Q_{j t_*} \nu)(V)} \sqrt{(P_{j t_*}\mu)(\varDelta)}
	\end{align}
	holds for any $\mu, \nu \in \mathcal{P}(S)$. 
	Here, for $t \geq 0$, $(P_t \mu)(V) \coloneqq \int_S P_t(x, V) \mu(dx)$ and $(Q_t \nu)(V)$ is defined similarly. 
\end{thm}

\begin{proof}
	Our proof is split into several steps. \\
	
	\noindent {\bf \underline{Step 1}} We may assume $\mu, \nu$ are delta distributions. \\
	
	\noindent According to Theorem 4.8 in \cite{villani2009optimal}, 
	\begin{align*}
		\mathbb{W}_{\tilde{d}}(P_{k t_*} \mu, Q_{k t_*} \nu) 
		\leq \int_{S \times S} \mathbb{W}_{\tilde{d}}(P_{k t_*}(x, \cdot), Q_{k t_*}(y, \cdot)) \pi(dx dy)
	\end{align*}
	holds for any $k \in \N$ and $\pi \in \Pi(\mu, \nu)$. 
	For any $\pi \in \Pi(\mu, \nu)$, the Schwarz's inequality yields 
	\begin{align*}
		&\int_{S \times S} \sqrt{1 + P_{j t_*}(x, V) + Q_{j t_*}(y, V)} \sqrt{P_{j t_*}(x, \varDelta)} \pi(dx dy) \\
		&\quad\leq \sqrt{1 + (P_{j t_*} \mu)(V) + (Q_{j t_*} \nu)(V)} \sqrt{(P_{j t_*}\mu)(\varDelta)}.
	\end{align*} 
	Therefore, if we show 
	\begin{align}
		\label{EQ_Markov_Kernel_Contraction_Dirac}
		\mathbb{W}_{\tilde{d}}(x, \cdot), Q_{k_0 t_*}(y, \cdot))
		\leq \frac{1}{2} \tilde{d}(x, y) + C \sum_{j=0}^{k_0-1} \sqrt{1 + P_{j t_*}(x, V) + Q_{j t_*}(y, V)} \sqrt{P_{j t_*}(x, \varDelta)}, 
	\end{align}
	which is (\ref{EQ_Markov_Kernel_Contraction}) when $\mu, \nu$ are delta distributions, the general case (\ref{EQ_Markov_Kernel_Contraction}) follows by taking the infimum with respect to $\pi$. \\
	
	\noindent {\bf \underline{Step 2}} Introducing $\tilde{d}_\delta$. \\
	
	\noindent Let $\delta > 0$ and $\tilde{d}_\delta(x, y) = \sqrt{d(x, y) (1 + \delta V(x) + \delta V(y))}$. 
	Thus, for each fixed $\delta > 0$, there exists constants $c_\delta = \Omega_\delta(1)$ and $C_\delta = O_\delta(1)$ such that $c_\delta \tilde{d}(x, y) \leq \tilde{d}_\delta(x, y) \leq C_\delta \tilde{d}(x, y)$ holds. 
	In the rest of proof, we will show that for a sufficiently small $\delta = \Omega_\alpha(1)$, we can take some $\tilde{\theta} = \Omega_{\alpha, \delta}(1)$ and $\tilde{C} = O_{\alpha, \delta}(1)$ so that $\tilde{\theta} \in (0, 1)$ and 
	\begin{align}
		\label{EQ_contracting_ineq_each_case}
		\mathbb{W}_{\tilde{d}_\delta}(P_{t_*}(x, \cdot), Q_{t_*}(y, \cdot)) 
		\leq \tilde{\theta} \tilde{d}_\delta(x, y) + \tilde{C} \sqrt{(1 + V(x) + V(y)) \varDelta(x)} 
	\end{align}
	holds. 
	As we will see in Step 3, (\ref{EQ_contracting_ineq_each_case}) is sufficient to prove (\ref{EQ_Markov_Kernel_Contraction_Dirac}). \\	
	
	\noindent {\bf \underline{Step 3}} (\ref{EQ_contracting_ineq_each_case}) indicates the following for any $k \in \N$. 
	\begin{align}
		\label{EQ_k_t_*_Ineq}
		\mathbb{W}_{\tilde{d}_\delta}(P_{k t_*}(x, \cdot), Q_{k t_*}(y, \cdot)) 
		\leq \tilde{\theta}^k \tilde{d}_\delta(x, y) 
		+ \tilde{C} \sum_{j=0}^{k-1} \sqrt{1 + P_{j t_*}(x, V) + Q_{j t_*}(y, V)} \sqrt{P_{j t_*}(x, \varDelta)}. 
	\end{align}
	In particular, if $\delta = \Omega_\alpha(1)$, we have (\ref{EQ_Markov_Kernel_Contraction_Dirac}). \\
	
	\noindent When $k=1$, (\ref{EQ_k_t_*_Ineq}) is clear by (\ref{EQ_contracting_ineq_each_case}). 
	Assume that (\ref{EQ_k_t_*_Ineq}) holds for $k = \ell-1$. 
	Then, for any $\pi \in \Pi(P_{(\ell-1) t_*}(x, \cdot), Q_{(\ell-1) t_*}(y, \cdot))$, we have
	\begin{align*}
		P_{\ell t_*}(x, \cdot) 
		= \int_S P_{t_*}(x^\prime, \cdot) P_{(\ell-1) t_*}(x, d x^\prime)
		= \int_{S \times S} P_{t_*}(x^\prime, \cdot) \pi(d x^\prime d y^\prime). 
	\end{align*} 
	Since the similar result also holds for $Q_{\ell t_*}(y, \cdot)$, by Theorem 4.8 in \cite{villani2009optimal}, we have
	\begin{align*}
		\mathbb{W}_{\tilde{d}_\delta}(P_{\ell t_*}(x, \cdot), Q_{\ell t_*}(y, \cdot))
		&\leq \int_{S \times S} \mathbb{W}_{\tilde{d}_\delta}(P_{t_*}(x^\prime, \cdot), Q_{t_*}(y^\prime, \cdot)) \pi(dx^\prime dy^\prime). 
	\end{align*}
	Therefore, the Schwarz's inequality yields 
	\begin{align*}
		&\mathbb{W}_{\tilde{d}_\delta}(P_{\ell t_*}(x, \cdot), Q_{\ell t_*}(y, \cdot)) \\
		&\quad\leq \tilde{\theta} \int_{S \times S} \tilde{d}_\delta(x^\prime, y^\prime) \pi(d x^\prime d y^\prime) 
		+ \tilde{C} \int_{S \times S} \sqrt{(1 + V(x^\prime) + V(y^\prime)) \varDelta(x^\prime)} \pi(d x^\prime d y^\prime) \\
		&\quad\leq \tilde{\theta} \int_{S \times S} \tilde{d}_\delta(x^\prime, y^\prime) \pi(d x^\prime d y^\prime) 
		+ \tilde{C} \sqrt{1 + P_{(\ell-1) t_*}(x, V) + Q_{(\ell-1) t_*}(\cdot, V)} \sqrt{P_{(k-1) t_*}(x, \varDelta)},  
	\end{align*}
	and taking the infimum with respect to $\pi$, we obtain 
	\begin{align*}
		\mathbb{W}_{\tilde{d}_\delta}(P_{\ell t_*}(x, \cdot), Q_{\ell t_*}(y, \cdot)) 
		&\leq \tilde{\theta} \mathbb{W}_{\tilde{d}_\delta}(P_{(\ell-1) t_*}(x, \cdot), Q_{(\ell-1) t_*}(y, \cdot)) \\
		&\quad+ \tilde{C} \sqrt{1 + P_{(\ell-1) t_*}(x, V) + Q_{(\ell-1) t_*}(\cdot, V)} \sqrt{P_{(k-1) t_*}(x, \varDelta)}. 
	\end{align*}
	Thus, (\ref{EQ_k_t_*_Ineq}) also holds for $k=\ell$. 
	
	$(C_\delta / c_\delta) \tilde{\theta}^{k_0} \leq 1/2$ holds for sufficiently large $k_0 = O_{\alpha, \delta}(1)$ since $\tilde{\theta} \in (0, 1)$. 
	Therefore, by taking $C = (C_\delta / c_\delta) \tilde{C} = O_{\alpha, \delta}(1)$, the pair $(k_0, C)$ satisfies (\ref{EQ_Markov_Kernel_Contraction}). 
	Hence, if we can take $\delta = \Omega_\alpha(1)$ so that (\ref{EQ_contracting_ineq_each_case}) holds, the proof is completed. \\
	
	\noindent {\bf \underline{Step 4}} How to take $\delta$, $\tilde{\theta}$ and $\tilde{C}$ so that (\ref{EQ_contracting_ineq_each_case}) holds; when $d(x, y) < 1$. \\
	
	\noindent For any $\pi \in \Pi(P_{t_*}(x, \cdot), Q_{t_*}(y, \cdot))$, the Schwarz's inequality indicates 
	\begin{align*}
		&\mathbb{W}_{\tilde{d}_\delta}(P_{t_*}(x, \cdot), Q_{t_*}(y, \cdot))^2 \\
		&\quad\leq \left( \int_{S \times S} d(x^\prime, y^\prime) \pi(d x^\prime d y^\prime) \right) 
		\times \left( \int_{S \times S} (1 + \delta V(x^\prime) + \delta V(y^\prime)) \pi(d x^\prime d y^\prime) \right).
	\end{align*}
	Here, since $\pi \in \Pi(P_{t_*}(x, \cdot), Q_{t_*}(y, \cdot))$ and $t_* > \log(8 C_V) / \gamma$, we have 
	\begin{align*}
		\int_{S \times S} (1 + \delta V(x^\prime) + \delta V(y^\prime)) \pi(d x^\prime d y^\prime)
		&= 1 + \delta P_{t_*}V (x) + \delta Q_{t_*}V (y) \\
		&\leq 1 + \frac{\delta}{8} \{ V(x) + V(y) \} + 2 \delta K_V. 
	\end{align*}
	Therefore, by $(\theta, \varDelta)$-contractivity of $d$, taking the infimum with respect to $\pi$, we obtain 
	\begin{align*}
		&\mathbb{W}_{\tilde{d}_\delta}(P_{t_*}(x, \cdot), Q_{t_*}(y, \cdot))^2 \\
		&\quad\leq \left( \theta d(x, y) + \varDelta(x) \right) \left( 1 + \frac{\delta}{8} \Big\{ V(x) + V(y) \Big\} + 2 \delta K_V \right) \\
		&\quad\leq \theta (1 + 2 \delta K_V) \tilde{d}_\delta(x, y)^2 + \left( 1 + \frac{\delta}{8} + 2 \delta K_V \right) (1 + V(x) + V(y)) \varDelta(x) 
	\end{align*}
	Hence, in this case, taking $\delta = \Omega_\alpha(1)$ so that $\tilde{\theta} = \sqrt{\theta (1 + 2 \delta K_V)} \in (0, 1)$ holds and $\tilde{C} = \sqrt{1 + \frac{\delta}{8} + 2 \delta K_V}$ is sufficient for (\ref{EQ_contracting_ineq_each_case}). \\
	
	\noindent {\bf \underline{Step 5}} How to take $\delta$, $\tilde{\theta}$ and $\tilde{C}$ so that (\ref{EQ_contracting_ineq_each_case}) holds; when $d(x, y) = 1$ and $V(x) + V(y) \geq 4 K_V $. \\
	
	\noindent In this case, we have
	\begin{align*}
		\tilde{d}_\delta(x, y)^2
		&= 1 + \delta (V(x) + V(y)) 
		\geq 1 + 3 \delta K_V + \frac{\delta}{4} (V(x) + V(y)).  
	\end{align*}
	Therefore, since $d \leq 1$ and $V$ is a common Lyapunov function for $P_t$ and $Q_t$, by the Schwarz's inequality, we obtain 
	\begin{align*}
		\mathbb{W}_{\tilde{d}_\delta}(P_{t_*}(x, \cdot), Q_{t_*}(y, \cdot))^2 
		&\leq \inf_{\pi \in \Pi(P_{t_*}(x, \cdot), Q_{t_*}(y, \cdot))} \int_{S \times S} d(x^\prime, y^\prime) (1 + \delta V(x^\prime) + \delta V(y^\prime)) \pi(d x^\prime d y^\prime) \\
		&\leq 1 + 2 \delta K_V + \frac{\delta}{8} (V(x) + V(y)) \\
		&\leq \max \left\{ \frac{1 + 2 \delta K_V}{1 + 3 \delta K_V}, \frac{1}{2} \right\} \left( 1 + 3 \delta K_V + \frac{\delta}{4} (V(x) + V(y)) \right) \\
		&\leq \max \left\{ \frac{1 + 2 \delta K_V}{1 + 3 \delta K_V}, \frac{1}{2} \right\} \tilde{d}_\delta(x, y)^2. 
	\end{align*}
	Hence, in this case, $\delta$ and $\tilde{C}$ can be chosen arbitrarily since $\tilde{\theta} = \sqrt{\max \left\{ \frac{1 + 2 \delta K_V}{1 + 3 \delta K_V}, \frac{1}{2} \right\}} \in (0, 1)$ holds always. \\
	
	\noindent {\bf \underline{Step 6}} How to take $\delta$, $\tilde{\theta}$ and $\tilde{C}$ so that (\ref{EQ_contracting_ineq_each_case}) holds; when $d(x, y) = 1$ and $V(x) + V(y) \leq 4 K_V $. \\
	
	\noindent Since $\{ V \leq 4 K_V \}$ is $\varepsilon$-$d$-small for $(P_{t_*}, Q_{t_*})$ and $d$ is lower semicontinuous, by Theorem 4.1 in \cite{villani2009optimal}, we can take some $\pi \in \Pi(P_{t_*}(x, \cdot), Q_{t_*}(y, \cdot))$ so that $\int_S d(x^\prime, y ^\prime) \pi(d x^\prime d y^\prime) \leq 1 - \varepsilon$ holds. 
	Thus, by the Schwarz's inequality, we obtain 
	\begin{align*}
		\mathbb{W}_{\tilde{d}_\delta}(P_{t_*}(x, \cdot), Q_{t_*}(y, \cdot))^2 
		&\leq (1 - \varepsilon) 
		\times \int_{S \times S} (1 + \delta V(x^\prime) + \delta V(y^\prime)) \pi(d x^\prime d y^\prime) \\
		&\leq (1 - \varepsilon) \left\{ 1 + 2 \delta K_V + \frac{\delta}{8} (V(x) + V(y)) \right\} \\
		&\leq (1 - \varepsilon) \left\{ 1 + \frac{5}{2} \delta K_V \right\} \\
		&\leq (1 - \varepsilon) \left\{ 1 + \frac{5}{2} \delta K_V \right\} \tilde{d}_\delta(x, y)^2. 
	\end{align*}
	Here, we used $\tilde{d}_\delta(x, y) \geq 1$ in the last inequality. 
	Hence, in this case, taking $\tilde{C}$ arbitrarily and $\delta = \Omega_\alpha(1)$ so that $\tilde{\theta} = \sqrt{(1 - \varepsilon) \left\{ 1 + \frac{5}{2} \delta K_V \right\}} \in (0, 1)$ holds is sufficient for (\ref{EQ_contracting_ineq_each_case}). 
\end{proof}

\begin{defi}
	\label{Def_Feller_Property}
	A Markov semigroup $\{ P_t \}_{t \in \mathcal{T}}$ on $S$ is said to satisfy the Feller property if $P_t(\cdot, f) \coloneqq \int_S f(y) P_t(\cdot, dy)$ defines a continuous function with respect to $\rho$ for any fixed $t \in \mathcal{T}$ and $f \in C_0(S; \R)$. 
	Here, $C_0(S; \R)$ is the set of all compactly supported and real valued continuous functions on $(S, \rho)$. 
\end{defi}

\begin{prop}
	\label{Thm_Existence_Invariant_Measure_Origin}
	Assume the same condition as Theorem \ref{Thm_Uniqueness_Invariant_Measure} with $P_t = Q_t$ and $\varDelta \equiv 0$. 
	Let a distance $d_0$ on $S$ satisfying $d_0 \leq \tilde{d}$ define the same topology as $\rho$. 
	Furthermore, let $\{ P_t \}_{t \in \mathcal{T}}$ satisfy the Feller property. 
	Then, there exist the unique invariant measures $\mu_*$ of $\{ P_t \}_{t \in \mathcal{T}}$ such that $V$ is integrable with respect to $\mu_*$. 
\end{prop}

\begin{proof}
	The existence and uniqueness of $\mu_*$ can be shown in a similar manner to the proof of Corollary 4.11 in \cite{hairer2011asymptotic}. 
\end{proof}

\section{Results on Ergodicity of Functional SDEs}
\label{SEC_Ergodicity_FSDE}

In this section, we extend the result of the ergodicity \cite{bao2020ergodicity} to the differences of two functional SDEs with different drift coefficients. 
In the following, we assume the same conditions as in Section \ref{SEC_Proof_Strategy} for the function $R : \R^d \to \R$ and functionals $\mathcal{H}, H : \mathcal{C}_r \to \R^d$. 
Let $\mathcal{X}^{(\xi)}$ and $X^{(\eta, \xi)}$ be the same as those in (\ref{EQ_FSDE_Conti}) and (\ref{EQ_FSDE_Disc}), respectively. 
Furthermore, for any $t \geq 0$ and nonnegative functional $f : \mathcal{C}_r \to [0, \infty)$, we define $\mathcal{P}_t$ and $P^{(\eta)}_t$ by $\mathcal{P}_t(\xi, f) = E[f(\mathcal{X}^{(\xi)}_t)]$ and $P^{(\eta)}_t(\xi, f) = E[f(X^{(\eta, \xi)}_t)]$, respectively. 
In the following, we check the conditions of Theorem \ref{Thm_Uniqueness_Invariant_Measure} for $\{ \mathcal{P}_t \}_{t \geq 0}$ and $\{ P^{(\eta)}_{k \eta} \}_{k=1}^\infty$, and give a sharp bound Theorem \ref{Thm_Generalizaiton_of_Bao} to the difference of them as an application of it. 
Theorem \ref{Thm_Generalizaiton_of_Bao} is used in Appendix \ref{SEC_Proof_Main} for the proof of Theorem \ref{Thm_General_Bound_on_Difference}. 

\subsection{Existence of Lyapunov functions}

\begin{lem}
	\label{Lem_Exi_Lyapunov_Function_Disc}
	(Discrete version of Proposition 1.3 in \cite{bao2020ergodicity})
	Let $p \geq 2$ and $\alpha = (p, K, m, b, M, r, d)$. 
	Then, if $m/3 > r$, then for some $\gamma = \Omega_\alpha(1)$ and $C = O_\alpha(1)$, $V_p(\xi) = \| \xi \|_r^p$ satisfies the following inequalities uniformly on $0 < \eta \leq 1$. 
	\begin{align*}
		\mathcal{P}_t(\xi, V_p) 
		\leq e^{- \gamma t} V_p(\xi) + C, \quad
		P^{(\eta)}_t(\xi, V_p) 
		\leq e^{- \gamma t} V_p(\xi) + C,\qquad \xi \in \mathcal{C}_r.
	\end{align*}
\end{lem}

\begin{proof}
	We only prove the inequality for $P^{(\eta)}_t$ since the other can be obtained by taking the limit $\eta \to 0$.  
	Then, as in the proof of Proposition 1.3 in \cite{bao2020ergodicity}, we only have to show 
	\begin{align}
		\label{EQ_Want_Lyapunov}
		E\left[\sup_{0 \leq s \leq t} (e^{p r s} \| X^{(\eta, \xi)}(s) \|_{\R^d}^p) \right] 
		\leq O_\alpha \left( \| \xi \|_r^p + e^{p r t} \right). 
	\end{align}
	
	Since $\nabla \| x \|_{\R^d}^p = p \| x \|_{\R^d}^{p-2} x$ and $\Delta \| x \|_{\R^d}^p = p (p+d-2) \| x \|_{\R^d}^{p-2}$, Ito's rule indicates
	\begin{align*}
		e^{p r t} \| X^{(\eta, \xi)}(t) \|_{\R^d}^p 
		&\leq \| \xi(0) \|_{\R^d}^p 
		+ p \sqrt{\frac{2}{\beta}} \int_0^t e^{p r s} \| X^{(\eta, \xi)}(s) \|_{\R^d}^{p-2} \langle X^{(\eta, \xi)}(s), dW(s) \rangle_{\R^d} \\
		&\quad+ p \int_0^t e^{p r s} \| X^{(\eta, \xi)}(s) \|_{\R^d}^{p-2} \left\{ r \| X^{(\eta, \xi)}(s) \|_{\R^d}^2 + K \| X^{(\eta)}(s) \|_{\R^d} \right. \\
		&\left.\quad- \langle X^{(\eta, \xi)}(s), \nabla R(X^{(\eta, \xi)}(\lfloor s / \eta \rfloor \eta)) \rangle_{\R^d} + \frac{p+d-2}{\beta} \right\} ds. 
	\end{align*}
	Here, by the $M$-smoothness and $(m, b)$-dissipativity of $R$, 
	\begin{align*}
		&- \langle X^{(\eta, \xi)}(s), \nabla R(X^{(\eta, \xi)}(\lfloor s / \eta \rfloor \eta)) \rangle_{\R^d} \\
		&\quad\leq - \langle X^{(\eta, \xi)}(s), \nabla R(X^{(\eta, \xi)}(s)) \rangle_{\R^d} + M \| X^{(\eta, \xi)}(s) \|_{\R^d} \| X^{(\eta, \xi)}(s) - X^{(\eta, \xi)}(\lfloor s / \eta \rfloor \eta) \|_{\R^d} \\
		&\quad\leq -m \| X^{(\eta, \xi)}(s) \|_{\R^d}^2 + b + M \| X^{(\eta, \xi)}(s) \|_{\R^d} \left(K \eta + \sqrt{\frac{2}{\beta}} \| W(s) - W(\lfloor s / \eta \rfloor \eta) \|_{\R^d}\right) 
	\end{align*}
	holds. 
	Furthermore, by the Young's inequality, we have
	\begin{align*}
		&M \sqrt{\frac{2}{\beta}} \| X^{(\eta, \xi)}(s) \|_{\R^d}^{p-1} \| W(s) - W(\lfloor s / \eta \rfloor \eta) \|_{\R^d} \\
		&\quad= \left( \frac{m}{3} \right)^{\frac{p-1}{p}} \| X^{(\eta, \xi)}(s) \|_{\R^d}^{p-1} \times M \sqrt{\frac{2}{\beta}} \left( \frac{3}{m} \right)^{\frac{p-1}{p}} \| W(s) - W(\lfloor s / \eta \rfloor \eta) \|_{\R^d} \\
		&\quad\leq \frac{m (p-1)}{3 p} \| X^{(\eta, \xi)}(s) \|_{\R^d}^p + \frac{M^p}{p} \left( \frac{2}{\beta} \right)^{p/2} \left( \frac{3}{m} \right)^{p-1} \| W(s) - W(\lfloor s / \eta \rfloor \eta) \|_{\R^d}^p.
	\end{align*}
	Therefore, by $\beta \geq 2/m$ and $m/3 > r$, 
	\begin{align}
		e^{p r t} \| X^{(\eta, \xi)}(t) \|_{\R^d}^p 
		&\leq \| \xi(0) \|_{\R^d}^p 
		+ p \sqrt{\frac{2}{\beta}} \int_0^t e^{p r s} \| X^{(\eta, \xi)}(s) \|_{\R^d}^{p-2} \langle X^{(\eta, \xi)}(s), dW(s) \rangle_{\R^d} \notag \\
		&\quad+ C_1 \int_0^t e^{p r s} \left( 1 + \| W(s) - W(\lfloor s / \eta \rfloor \eta) \|_{\R^d}^p \right) ds  \label{EQ_Lambda_Pointwise_Bound}
	\end{align}
	holds for some $C_1 = O_\alpha(1)$. 
	Here, the Burkholder-Davis-Gundy inequality indicates
	\begin{align*}
		&p \sqrt{\frac{2}{\beta}} E\left[ \sup_{0 \leq s \leq t} \left| \int_0^s e^{p r u} \| X^{(\eta, \xi)}(u) \|_{\R^d}^{p-2} \langle X^{(\eta, \xi)}(u), dW(u) \rangle_{\R^d} \right| \right] \\
		&\quad\leq C_2 E\left[ \left\{ \int_0^t e^{2 p r s} \| X^{(\eta, \xi)}(s) \|_{\R^d}^{2p - 2} ds \right\}^{1/2} \right] \\
		&\quad\leq C_2 E\left[ \left\{ \left( \sup_{0 \leq s \leq t} (e^{p r s} \| X^{(\eta, \xi)}(s) \|_{\R^d}^p) \right) \int_0^t e^{p r s} \| X^{(\eta, \xi)}(s) \|_{\R^d}^{p-2} ds \right\}^{1/2} \right] \\
		&\quad\leq \frac{1}{2} E\left[ \sup_{0 \leq s \leq t} (e^{p r s} \| X^{(\eta, \xi)}(s) \|_{\R^d}^p) \right] 
		+ \frac{C_2^2}{2} \int_0^t e^{p r s} E[\| X^{(\eta, \xi)}(s) \|_{\R^d}^{p-2}] ds   
	\end{align*}
	for some $C_2 = O_\alpha(1)$. 
	Applying (\ref{EQ_Lambda_Pointwise_Bound}) for $p-2$, we can take some $C_3 = O_\alpha(1)$ so that
	\begin{align*}
		e^{(p-2) r s} E[\| X^{(\eta, \xi)}(s) \|_{\R^d}^{p-2}] 
		&\leq \| \xi(0) \|_{\R^d}^{p-2} 
		+ C_3 \int_0^s e^{(p-2) r u} \left\{ 1 + \eta^{(p-2) / 2} \right\} du \\
		&\leq \| \xi(0) \|_{\R^d}^{p-2} + \frac{2 C_3}{(p-2) r} e^{(p-2) r s}
	\end{align*}
	holds.
	As a result, we obtain  
	\begin{align*}
		&p \sqrt{\frac{2}{\beta}} E\left[ \sup_{0 \leq s \leq t} \left| \int_0^s e^{p r u} \| X^{(\eta, \xi)}(u) \|_{\R^d}^{p-2} \langle X^{(\eta, \xi)}(u), dW(u) \rangle_{\R^d} \right| \right] \\
		&\quad\leq \frac{1}{2} E\left[ \sup_{0 \leq s \leq t} (e^{p r s} \| X^{(\eta, \xi)}(s) \|_{\R^d}^p) \right] + \frac{C_2^2}{2} \left( \frac{1}{2 r} \| \xi(0) \|_r^{p-2} e^{2 r t} + \frac{2 C_3}{p (p-2) r^2} e^{p r t} \right). 
	\end{align*}
	Since the Young's inequality indicates $\| \xi(0) \|_r^{p-2} e^{2 r t} \leq \frac{p-2}{p} \| \xi(0) \|_r^p + \frac{2}{p} e^{p r t}$, taking the supremum and expectation in both sides of (\ref{EQ_Lambda_Pointwise_Bound}), we obtain (\ref{EQ_Want_Lyapunov}). 
\end{proof}

\subsection{$\rho_{r, \delta}$-smallness of lower level sets of Lyapunov functions}

For $\xi , \xi^\prime \in \mathcal{C}_r$, we set 
\begin{align}
	\label{EQ_rho_r}
	\rho_r(\xi, \xi^\prime) 
	= \| \xi - \xi^\prime \|_r. 
\end{align}
Furthermore, for fixed $\delta > 0$ and $\Gamma > 0$, we define $\rho_{r, \delta}$, $B_\Gamma$ and $t_{\Gamma, \delta}$ by
\begin{align*}
	\rho_{r, \delta} 
	= 1 \wedge (\delta^{-1} \rho_r),\quad 
	B_\Gamma 
	= \{ \xi \in \mathcal{C}_r \mid \| \xi \|_r \leq \Gamma \}
\end{align*}
and
\begin{align}
	\label{EQ_t_R_delta}
	t_{\Gamma, \delta} 
	= 1 + \frac{1}{2 r} \log \left\{ \frac{3}{\delta^2} \left( \Gamma + \frac{\delta}{3} \right)^2 \right\} 
	= O_{\Gamma, \delta, r}(1),  
\end{align}
respectively. 

\begin{lem}
	\label{Lem_Asymptotic_Boundedness}
	(Discrete version of Lemma 2.2 in \cite{bao2020ergodicity})
	Let $\delta > 0$, $\Gamma > 0$, $T > 0$ and $\alpha = (\delta, T, \| \nabla R(0) \|_{\R^d}, K, m, M, r, d)$. 
	Then, 
	\begin{align*}
		\inf_{t_{\Gamma, \delta} \leq u \leq t_{\Gamma, \delta} + T} \inf_{\xi \in B_\Gamma} P(X_u^{(\eta, \xi)} \in B_\delta) 
		\geq \Omega_\alpha(1)
	\end{align*}
	holds uniformly on $0 < \eta \leq 1$. 
	The same result holds also for $\mathcal{X}^{(\xi)}$. 
\end{lem}

\begin{proof}
	For fixed $\xi \in B_\Gamma$ and $\delta > 0$, we can take a function $h \in C_0^\infty([0, \infty); \R^d)$ such that 
	\begin{align}
		\label{EQ_Condition_function_h}
		h(0) 
		= \xi(0) - \frac{\delta}{3} (1, 0 \dots, 0)^\top,\quad 
		\| h \|_{\R^d} \leq \| h(0) \|_{\R^d},\quad 
		h(s) = 0, \, s \geq 1.
	\end{align}
	Here, $C_b^\infty([0, \infty); \R^d)$ denotes the set of all compactly supported and $\R^d$-valued smooth functions on $[0, \infty)$. 
	For the process
	\begin{align}
		\label{EQ_Radial_Process}
		D(s) 
		\coloneqq \| X^{(\eta, \xi)}(s) - h(s) \|_{\R^d}^2 - \frac{\delta^2}{18},\quad s \geq 0, 
	\end{align}
	we define the stopping time $\tau$ by
	\begin{align}
		\label{EQ_Radial_Process_Stopping}
		\tau 
		= \inf \left\{ s \geq 0 \,\Big|\, e^{2 r s} |D(s)| \geq \frac{\delta^2}{18} \right\}. 
	\end{align}
	Then, for any $t > 0$, as in the proof of Lemma 2.2 in \cite{bao2020ergodicity}, we can take $c = \Omega_{\alpha, t}(1)$ so that 
	\begin{align*}
		Y(s, t) 
		= e^{2 r (s \wedge (t \wedge \tau))} D(s \wedge (t \wedge \tau)) + \mathds{1}_{\{ s > (t \wedge \tau) \}} (W^1(s) - W^1((t \wedge \tau))),\quad s \geq 0
	\end{align*}
	satisfies 
	\begin{align*}
		\inf_{t \leq u \leq t + T} P\left( \sup_{0 \leq s \leq u} \| Y(s, u) \|_{\R^d}^2 < \frac{\delta^2}{18} \right) 
		\geq c, 
	\end{align*}
	where $W^1$ denotes the first coordinate of $W$. 
	Therefore, by the definition of $\tau$, 
	\begin{align*}
		\inf_{t \leq u \leq t + T} P\left( \sup_{0 \leq s \leq u} (e^{2 r s} |D(s)|)^2 < \frac{\delta^2}{18} \right)
		&= \inf_{t \leq u \leq t + T} P\left( \sup_{0 \leq s \leq u} \| Y(s, u) \|_{\R^d}^2 < \frac{\delta^2}{18} \right) 
		\geq c 
	\end{align*}
	holds. 
	The rest of proof is quite similar to that of Lemma 2.2 in \cite{bao2020ergodicity} and is omitted here. 
\end{proof}

\begin{lem}
	\label{Lem_d-smallness}
	(Discrete version of Lemma 2.4 in \cite{bao2020ergodicity})
	For any fixed $\delta > 0$, $\Gamma > 0$ and $t > 0$, 
	\begin{align*}
		\kappa_t 
		= \min \left\{ \inf_{\xi^\prime \in B_\Gamma} P(\mathcal{X}^{(\xi^\prime)}_t \in B_{\delta / 4}),\ \inf_{\xi \in B_\Gamma} P(X^{(\eta, \xi)}_t \in B_{\delta / 4}) \right\} 
	\end{align*}
	satisfies  
	\begin{align*}
		\mathbb{W}_{\rho_{r, \delta}}(\mathcal{P}_t(\xi^\prime, \cdot), P^{(\eta)}_t(\xi, \cdot)) 
		\leq 1 - \frac{\kappa_t^2}{2},\quad \xi, \xi^\prime \in B_\Gamma. 
	\end{align*}
\end{lem}

\begin{proof}
	The proof is quite similar to that of Lemma 2.4 in \cite{bao2020ergodicity} and is omitted. 
\end{proof}

\subsection{$(\theta, \varDelta)$-contractivity of $\rho_{r, \delta}$ for $(\mathcal{P}_t, P^{(\eta)}_t)$}

\begin{lem}
	\label{Lem_Asymptotic_Coupling_Bound}
	(Generalization of Lemma 3.3 in \cite{bao2019asymptotic})
	For each $\xi^\prime \in \mathcal{C}_r$ and $\lambda > 0$, we define the process $\mathcal{Y}^{(\xi^\prime)}$ as the solution with an initial value $\xi^\prime$ of 
	\begin{align}
		d\mathcal{Y}(t) 
		&= \mathcal{H}(\mathcal{Y}_t) dt - \nabla R(\mathcal{Y}(t)) dt + \lambda ( X^{(\eta, \xi)}(t) - \mathcal{Y}(t) ) dt + \sqrt{\frac{2}{\beta}} dW(t). \label{EQ_Asymptotic_Coupling}
	\end{align}
	Suppose that $r_0 \in (0, r)$ and $p \geq 1$ are given. 
	Then, for $\alpha = (p, r_0, \| \nabla R(0) \|_{\R^d}, K, m, b, M, r, d)$, there exist some $\lambda = O_\alpha(1)$ and $C = O_\alpha(1)$ such that 
	\begin{align}
		E[\| X^{(\eta, \xi)}_t - \mathcal{Y}^{(\xi^\prime)}_t \|_r^p] 
		&\leq C e^{- p r_0 t} \| \xi - \xi^\prime \|_r^p 
		+ C e^{p (r - r_0) t} \eta^{p/2} (1 + \| \xi \|_r^p) \notag \\
		&\quad+ C e^{p (r - r_0) t} \int_0^t e^{- p r (t-s)} E[ \| H(X^{(\eta, \xi)}_s) - \mathcal{H}(X^{(\eta, \xi)}_s) \|_{\R^d}^p ] ds 
		\label{EQ_Asymptotic_Coupling_Bound}
	\end{align}
	holds for any $\xi, \xi^\prime \in \mathcal{C}_r$ and $t \geq 0$. 
\end{lem}

\begin{proof} Let $Z(t) = X^{(\eta, \xi)}(t) - \mathcal{Y}^{(\xi^\prime)}(t)$. 
	Then, we have 
	\begin{align*}
		Z(t) 
		&= \xi(0) - \xi^\prime(0) + \int_0^t \left\{ H(X^{(\eta, \xi)}_{\lfloor s / \eta \rfloor \eta}) - \mathcal{H}(\mathcal{Y}^{(\xi^\prime)}_t) - \nabla R(X^{(\eta, \xi)}(\lfloor s / \eta \rfloor \eta)) + \nabla R(\mathcal{Y}^{(\xi^\prime)}(s)) \right. \\
		&\left.\quad- \lambda (X^{(\eta, \xi)}(s) - \mathcal{Y}^{(\xi^\prime)}(s)) \right\} ds,  
	\end{align*}
	and thus by $\| Z(t) \|_{\R^d} \leq \| Z_t \|_r$,
	\begin{align}
		d \| Z(t) \|_{\R^d}^2 
		&= 2 \langle Z(t), H(X^{(\eta, \xi)}_{\lfloor s / \eta \rfloor \eta}) - \mathcal{H}(\mathcal{Y}^{(\xi^\prime)}_t) \rangle_{\R^d} dt 
		- 2 \lambda \langle Z(t), X^{(\eta, \xi)}(t) - \mathcal{Y}^{(\xi^\prime)}(t) \rangle_{\R^d} dt \notag \\
		&\quad- 2 \langle Z(t), \nabla R(X^{(\eta, \xi)}(\lfloor t / \eta \rfloor \eta)) - \nabla R(\mathcal{Y}^{(\xi^\prime)}(t)) \rangle_{\R^d} dt \notag \\
		&\leq 2 (K + M - \lambda) \| Z_t \|_r^2 dt 
		+ 2 \| Z_t \|_r \{ \| H(X^{(\eta, \xi)}_{\lfloor t / \eta \rfloor \eta}) - \mathcal{H}(X^{(\eta, \xi)}_t) \|_{\R^d} \notag \\
		&\quad+ M \| X^{(\eta, \xi)}(\lfloor t / \eta \rfloor \eta) - X^{(\eta, \xi)}(t) \|_{\R^d} \} dt \notag \\
		&\leq \left\{ 3 (M + K) + 1 - 2 \lambda \right\} \| Z_t \|_r^2 dt 
		+ \| H(X^{(\eta, \xi)}_t) - \mathcal{H}(X^{(\eta, \xi)}_t) \|_{\R^d}^2 dt \notag \\ 
		&\quad+ (K + M) \| X^{(\eta, \xi)}_{\lfloor t / \eta \rfloor \eta} - X^{(\eta, \xi)}_t \|_{\R^d}^2 dt. \label{EQ_Lambda_to_0}
	\end{align}
	Therefore, since 
	\begin{align*}
		d (e^{2 \lambda t} \| Z(t) \|_{\R^d}^2) 
		&= 2 \lambda e^{2 \lambda t} \| Z(t) \|_{\R^d}^2 dt + e^{2 \lambda t} d \| Z(t) \|_{\R^d}^2 dt \\
		&\leq \left\{ 3 (M + K) + 1 \right\} e^{2 \lambda t} \| Z_t \|_r^2 dt 
		+ e^{2 \lambda t} \| H(X^{(\eta, \xi)}_t) - \mathcal{H}(X^{(\eta, \xi)}_t) \|_{\R^d}^2 dt \\
		&\quad+ (K + M) e^{2 \lambda t} \| X^{(\eta, \xi)}_{\lfloor t / \eta \rfloor \eta} - X^{(\eta, \xi)}_t \|_{\R^d}^2 dt, 
	\end{align*}
	we can take some $C_1 = O_{K, M}(1)$ as in the proof of Lemma 3.3 in \cite{bao2019asymptotic} so that 
	\begin{align*}
		e^{2 r t} \| Z_t \|_r^2 
		&\leq 2 \| Z_0 \|_r^2 
		+ \int_0^t e^{- \kappa (t-s)} e^{2 r s} \| H(X^{(\eta, \xi)}_s) - \mathcal{H}(X^{(\eta, \xi)}_s) \|_{\R^d}^2 ds \\
		&\quad + C_1 \int_0^t e^{- \kappa (t-s)} e^{2 r s} \| Z_s \|_r^2 ds + C_1 \int_0^t e^{- \kappa (t-s)} e^{2 r s} \| X^{(\eta, \xi)}_{\lfloor s / \eta \rfloor \eta} - X^{(\eta, \xi)}_s \|_{\R^d}^2 ds 
	\end{align*}
	holds, where $\kappa = 2 (\lambda - r) > 0$. 
	In particular, setting $C_2 = (4 \max \{ C_1, 2 \})^{p/2}$, we obtain 
	\begin{align*}
		E[e^{p r t} \| Z_t \|_r^p] 
		&\leq C_2 \left\{ \| \xi - \eta \|_r^p + E\left[ \left| \int_0^t e^{- \kappa (t-s)} e^{2 r s} \| Z_s \|_r^2 ds \right|^{p/2} \right] \right. \\ 
		&\quad\left. + E\left[ \left| \int_0^t  e^{- \kappa (t-s)} e^{2 r s} \| H(X^{(\eta, \xi)}_s) - \mathcal{H}(X^{(\eta, \xi)}_s) \|_{\R^d}^2 ds \right|^{p / 2} \right] \right. \\
		&\quad\left. + E\left[ \left| \int_0^t e^{- \kappa (t-s)} e^{2 r s} \| X^{(\eta, \xi)}_{\lfloor s / \eta \rfloor \eta} - X^{(\eta, \xi)}_s \|_{\R^d}^2 ds \right|^{p/2} \right] \right\}. 
	\end{align*}
	Since the H\"{o}lder's inequality yields the bound of the form
	\begin{align*}
		\left| \int_0^t e^{- \kappa (t-s)} e^{2 r s} \| Z_s \|_r^2 ds \right|^{p/2} 
		\leq \left( \int_0^\infty e^{-\frac{\kappa p s}{p - 2}} ds \right)^{\frac{p-2}{p}} \int_0^t e^{p r s} \| Z_s \|_r^p ds
	\end{align*}
	for each term in R.H.S., the aforementioned result indicates  
	\begin{align*}
		e^{p r t} E[ \| Z_t \|_r^p ] 
		&\leq C_2 \| \xi - \eta \|_r^p 
		+ C_2 \left( \int_0^\infty e^{-\frac{\kappa p s}{p - 2}} ds \right)^{\frac{p-2}{p}} \int_0^t e^{p r s} E[ \| Z_s \|_r^p ] ds \\
		&\quad+ C_2 \left( \int_0^\infty e^{-\frac{\kappa p s}{p - 2}} ds \right)^{\frac{p-2}{p}} \left( \int_0^t e^{p r s} E[\| H(X^{(\eta, \xi)}_s) - \mathcal{H}(X^{(\eta, \xi)}_s) \|_{\R^d}^p] ds \right. \\
		&\left.\quad+ \int_0^t e^{p r s} E[\| X^{(\eta, \xi)}_{\lfloor s / \eta \rfloor \eta} - X^{(\eta, \xi)}_s \|_{\R^d}^p] ds \right).  
	\end{align*}
	Therefore, by the Gronwall's inequality, we obtain the following for $c(\lambda) = C_2 \left( \int_0^\infty e^{-\frac{\kappa p s}{p - 2}} ds \right)^{\frac{p-2}{p}}$. 
	\begin{align*}
		e^{p r t} E[ \| Z_t \|_r^p ] 
		&\leq e^{c(\lambda) t} \left\{ C_2 \| \xi - \eta \|_r^p + c(\lambda) \left( \int_0^t e^{p r s} E[\| H(X^{(\eta, \xi)}_s) - \mathcal{H}(X^{(\eta, \xi)}_s) \|_{\R^d}^p] ds \right. \right. \\
		&\left. \left. \quad+ \int_0^t e^{p r s} E[\| X^{(\eta, \xi)}_{\lfloor s / \eta \rfloor \eta} - X^{(\eta, \xi)}_s \|_{\R^d}^p] ds \right) \right\}. 
	\end{align*}
	By taking $\lambda > r$ sufficiently large so that $c(\lambda) < (r - r_0) p$ holds, this yields 
	\begin{align*}
		E[ \| Z_t \|_r^p ]  
		&\leq C_2 e^{- p r_0 t} \| \xi - \eta \|_r^p + c(\lambda) e^{p (r - r_0) t} \left( \int_0^t e^{p r (s-t)} E[\| H(X^{(\eta, \xi)}_t) - \mathcal{H}(X^{(\eta, \xi)}_t) \|_{\R^d}^p] ds \right. \\
		&\left. \quad+ e^{p (r - r_0) t} \int_0^t e^{p r (s-t)} E[\| X^{(\eta, \xi)}_{\lfloor s / \eta \rfloor \eta} - X^{(\eta, \xi)}_s \|_{\R^d}^p] ds \right). 
	\end{align*}
	Since we have $\sup_{s \geq 0} E[\| X^{(\eta, \xi)}_{\lfloor s / \eta \rfloor \eta} - X^{(\eta, \xi)}_s \|_{\R^d}^p] \leq O_\alpha(\eta^{p/2} (1 + \| \xi \|_r^p))$ by Lemma \ref{Lem_Exi_Lyapunov_Function_Disc}, (\ref{EQ_Asymptotic_Coupling_Bound}) holds for some $C = O_\alpha(1)$. 
\end{proof}

\begin{lem}
	\label{Lem_Alpha_Contracting} 
	(Generalization of Lemma 2.5 in \cite{bao2020ergodicity}) 
	Let $\theta \in (0, 1)$ and $r_0 \in (0, r)$ be given, and let $\alpha = (\theta, r_0, \| \nabla R(0) \|_{\R^d}, K, m, b, M, r, d)$. 
	Then, there exist some $\delta = \Omega_\alpha(\beta^{-1/2})$, $t = O_\alpha(1)$ and $C = O_\alpha(1)$ such that 
	\begin{align*}
		&\mathbb{W}_{\rho_{r, \delta}}(\mathcal{P}_{k \eta}(\xi^\prime, \cdot), P^{(\eta)}_{k \eta}(\xi, \cdot)) \\
		&\quad\leq \theta \rho_{r, \delta}(\xi, \xi^\prime) 
		+ C (1 + \beta^{1/2}) e^{C k \eta} \left\{ \eta^{1/2} (1 + \| \xi \|_r) 
		+ \sup_{s \geq 0} E[\| H(X^{(\eta, \xi)}_s) - \mathcal{H}(X^{(\eta, \xi)}_s) \|_{\R^d}^2]^{1/2} \right\} 
	\end{align*}
	holds uniformly on $t \leq k \eta \leq 2 t$, $0 < \eta \leq 1$ and $\xi, \xi^\prime \in \mathcal{C}_r$ with $\rho_{r, \delta}(\xi, \xi^\prime) < 1$.
\end{lem}

\begin{proof}
	If we define $\mathcal{Y}^{(\xi^\prime)}$ by (\ref{EQ_Asymptotic_Coupling}), then setting $\lambda = O_\alpha(1)$ and $C_1 = O_\alpha(1)$ to be sufficiently large, 
	\begin{align}
		E[\| X_t^{(\eta, \xi)} - \mathcal{Y}_t^{(\xi^\prime)} \|_r^2] 
		&\leq C_1 e^{- 2 r_0 t} \| \xi - \xi^\prime \|_r^2 
		+ C_1 e^{2 (r - r_0) t} \left\{ \eta (1 + \| \xi \|_r^2) \right. \notag \\
		&\left.\quad+ \int_0^t e^{2 r (s-t)} E[\| H(X^{(\eta, \xi)}_s) - \mathcal{H}(X^{(\eta, \xi)}_s) \|_{\R^d}^2] ds \right\} \label{EQ_X_Y_Diff_Exp_Decay} 
	\end{align}
	holds by Lemma \ref{Lem_Asymptotic_Coupling_Bound}. 
	Furthermore, if $\lambda > 0$ is sufficiently large, integrating both sides of (\ref{EQ_Lambda_to_0}) over $[0, t]$, we obtain
	\begin{align}
		&E[\| X_{t \wedge \tau}^{(\eta, \xi)} - \mathcal{Y}_{t \wedge \tau}^{(\xi^\prime)} \|_r^2] \notag \\
		&\quad\leq C_1 \left( \| \xi - \xi^\prime \|_r^2 + \eta t (1 + \| \xi \|_r^2) 
		+ \int_0^t E[\| H(X^{(\eta, \xi)}_s) - \mathcal{H}(X^{(\eta, \xi)}_s) \|_{\R^d}^2] ds 
		\right) \label{EQ_X_Y_Stopping_Diff} 
	\end{align}
	for any stopping time $\tau$. 
	We set $h(s) = \lambda \sqrt{\beta / 2} (X^{(\eta, \xi)}(s) - \mathcal{Y}^{(\xi^\prime)}(s))$ and 
	\begin{align*}
		Z_t 
		= \exp \left( - \int_0^t \langle h(s), dW(s) \rangle_{\R^d} - \frac{1}{2} \int_0^t \| h(s) \|_{\R^d}^2 ds \right),  
	\end{align*}
	and to use the Girsanov's theorem, for a small positive $\varepsilon \in (0, 1)$, we set 
	\begin{align}
		\label{EQ_tau_epsilon}
		\tau_\varepsilon 
		= \inf \left\{ t \geq 0 \,\Big|\, \int_0^t \| h(s) \|_{\R^d}^2 ds \geq \frac{1}{\varepsilon} \| \xi - \xi^\prime \|_r^2 \right\}. 
	\end{align}
	Then, as in the proof of Lemma 2.5 in \cite{bao2020ergodicity}, if we set $d Q_{t, \varepsilon} = Z_{t \wedge \tau_\varepsilon} d P$, then the solution $\tilde{\mathcal{Y}}$ with an initial value $\xi^\prime$ of 
	\begin{align}
		\label{EQ_SDE_for_Y_tilder}
		d \tilde{\mathcal{Y}}(s) 
		= \left\{ \mathcal{H}(\tilde{\mathcal{Y}}_s) - \nabla R(\tilde{\mathcal{Y}}(s)) + \mathds{1}_{\{ \tau_\varepsilon \geq s \}} \lambda (X^{(\eta, \xi)}(s) - \tilde{\mathcal{Y}}(s)) \right\} ds 
		+ \sqrt{\frac{2}{\beta}} dW(s)
	\end{align}
	satisfies 
	\begin{align*}
		P(\mathcal{X}^{(\xi^\prime)}_t \in \cdot) 
		= Q_{t, \varepsilon}(\tilde{\mathcal{Y}}^{(\xi^\prime)}_t \in \cdot),\quad t \geq 0, 
	\end{align*}
	and 
	\begin{align*}
		&\mathbb{W}_{\rho_{r, \delta}}(\mathcal{P}_t(\xi^\prime, \cdot), P^{(\eta)}_t(\xi, \cdot)) \\
		&\quad\leq E[\mathds{1}_{\{ t \leq \tau_\varepsilon \}} \rho_{r, \delta}(X_t^{(\eta, \xi)}, \tilde{\mathcal{Y}}_t^{(\xi^\prime)})] 
		+ E[\mathds{1}_{\{ t > \tau_\varepsilon \}} \rho_{r, \delta}(X_t^{(\eta, \xi)}, \tilde{\mathcal{Y}}_t^{(\xi^\prime)})] 
		+ E[(Z_{t \wedge \tau_\varepsilon} - 1)^+] \\
		&\quad\eqqcolon I_1(t) + I_2(t) + I_3(t)
	\end{align*}
	holds. 
	In the following, we retake $C_1$ to be larger line to line if necessary. 
	
	First, according to (\ref{EQ_X_Y_Diff_Exp_Decay}), if $\xi, \xi^\prime \in \mathcal{C}_r$ satisfy $\rho_{r, \delta}(\xi, \xi^\prime) < 1$, then we have
	\begin{align*}
		I_1(t) 
		&\leq \delta^{-1} E[\| X_t^{(\eta, \xi)} - \tilde{\mathcal{Y}}_t^{(\xi^\prime)} \|_r] \\
		&\leq C_1 e^{- r_0 t} \rho_{r, \delta}(\xi, \xi^\prime) 
		+ \frac{C_1 e^{2 (r - r_0) t}}{\delta} \left\{ \eta^{1/2} (1 + \| \xi \|_r) \right. \\
		&\left.\quad + \left( \int_0^t e^{2 r (s-t)} E[\| H(X^{(\eta, \xi)}_s) - \mathcal{H}(X^{(\eta, \xi)}_s) \|_{\R^d}^2] ds  \right)^{1/2}
		\right\}. 
	\end{align*}
	
	Second, by the Markov property of $\{ (X_{k \eta}^{(\eta, \xi)}, \tilde{\mathcal{Y}}_{k \eta}^{(\xi^\prime)}) \}_{k=0}^\infty$, 
	\begin{align*}
		I_2(k \eta) 
		&\leq \delta^{-1} E[\mathds{1}_{\{ k \eta > \tau_\varepsilon \}} \| X_{k \eta}^{(\eta, \xi)} - \tilde{\mathcal{Y}}_{k \eta}^{(\xi^\prime)} \|_r] \\
		&= \delta^{-1} E\left[ \mathds{1}_{\{ k \eta > \tau_\varepsilon \}} E\left[ \| X_{k \eta}^{(\eta, \xi)} - \tilde{\mathcal{Y}}_{k \eta}^{(\xi^\prime)} \|_r \,|\, \F_{\lfloor \tau_\varepsilon / \eta \rfloor \eta} \right] \right] \\
		&= \delta^{-1} E\left[ \mathds{1}_{\{ k \eta > \tau_\varepsilon \}} E\left[ \| X_{k \eta-s}^{(\eta, \bar{\xi})} - \tilde{\mathcal{Y}}_{k \eta-s}^{(\bar{\xi^\prime})} \|_r \right] \Big|_{(\bar{\xi}, \bar{\xi}^\prime, s) = (X_s^{(\eta, \xi)}, \mathcal{Y}_s^{(\xi^\prime)}, \lfloor \tau_\varepsilon / \eta \rfloor \eta)} \right] 
	\end{align*}
	holds. 
	Here, as a simple consequence of the Gronwall's inequality, we have 
	\begin{align*}
		E\left[ \| X_{k \eta-s}^{(\eta, \bar{\xi})} - \tilde{\mathcal{Y}}_{k \eta-s}^{(\bar{\xi^\prime})} \|_r \right] 
		\leq C_1 e^{C_1 (k \eta -s)} \left\{ \eta^{1/2} (1 + \| \bar{\xi} \|_r) + \| \bar{\xi} - \bar{\xi^\prime} \|_r \right\}. 
	\end{align*}
	Therefore, since $\mathcal{Y}_{t \wedge \tau_\varepsilon}^{(\xi^\prime)} = \tilde{\mathcal{Y}}_{t \wedge \tau_\varepsilon}^{(\xi^\prime)}$ holds for each $t$, by Lemma \ref{Lem_Exi_Lyapunov_Function_Disc} and (\ref{EQ_X_Y_Stopping_Diff}), 
	\begin{align*}
		I_2(k \eta) 
		&\leq \frac{C_1 e^{C_1 k \eta}}{\delta} \left\{ E\left[ \mathds{1}_{\{ k \eta > \tau_\varepsilon \}} \| X_{\lfloor \tau_\varepsilon / \eta \rfloor \eta}^{(\eta, \xi)} - \tilde{\mathcal{Y}}_{\lfloor \tau_\varepsilon / \eta \rfloor \eta}^{(\xi^\prime)} \|_r \right] + \eta^{1/2} (1 + \| \xi \|_r) \right\} \\
		&\leq \frac{C_1 e^{C_1 k \eta}}{\delta} \left\{ P(k \eta > \tau_\varepsilon)^{1/2} E\left[ \| X_{(k \eta) \wedge (\lfloor \tau_\varepsilon / \eta \rfloor \eta)}^{(\eta, \xi)} - \mathcal{Y}_{(k \eta) \wedge (\lfloor \tau_\varepsilon / \eta \rfloor \eta)}^{(\xi^\prime)} \|_r^2 \right]^{1/2} + \eta^{1/2} (1 + \| \xi \|_r) \right\} \\
		&\leq \frac{C_1 e^{C_1 k \eta}}{\delta} \left\{ P(k \eta > \tau_\varepsilon)^{1/2} \| \xi - \xi^\prime \|_r 
		+ \eta^{1/2} (1 + \| \xi \|_r) \right. \\
		&\left.\quad+ \left( \int_0^{k \eta} E[\| H(X^{(\eta, \xi)}_s) - \mathcal{H}(X^{(\eta, \xi)}_s) \|_{\R^d}^2] ds  \right)^{1/2}
		\right\}
	\end{align*}
	holds for $\xi, \xi^\prime \in \mathcal{C}_r$ with $\rho_{r, \delta}(\xi, \eta) < 1$. 
	Furthermore, for each $t$, by (\ref{EQ_X_Y_Diff_Exp_Decay}), we have 
	\begin{align}
		P(t > \tau_\varepsilon) 
		&\leq P \left( \int_0^t \| h(s) \|_{\R^d}^2 ds \geq \frac{1}{\varepsilon} \| \xi - \xi^\prime \|_r^2 \right) \notag \\
		&\leq \frac{\varepsilon}{\| \xi - \xi^\prime \|_r^2} \int_0^t E[\| h(s) \|_{\R^d}^2] ds \notag \\
		&\leq \frac{\lambda^2 \beta \varepsilon}{2 \| \xi - \xi^\prime \|_r^2} \int_0^t E[\| X^{(\eta, \xi)}(s) - \mathcal{Y}^{(\xi^\prime)}(s) \|_r^2] ds \notag \\
		&\leq \frac{C_1 \lambda^2 \beta \varepsilon}{2 \| \xi - \xi^\prime \|_r^2} \int_0^t \left\{ e^{- 2 r_0 s} \| \xi - \xi^\prime \|_r^2 
		+ e^{2 (r - r_0) s} \eta (1 + \| \xi \|_r^2) \right. \notag \\
		&\left.\quad+ e^{2 (r - r_0) s} \int_0^s e^{2 r (u-s)} E[\| H(X^{(\eta, \xi)}_u) - \mathcal{H}(X^{(\eta, \xi)}_u) \|_{\R^d}^2] du 
		\right\} ds \notag \\
		&\leq \frac{C_1 \lambda^2 \beta \varepsilon}{2 r_0 \| \xi - \xi^\prime \|_r^2} \left\{ \| \xi - \xi^\prime \|_r^2 
		+ \frac{e^{2 (r - r_0) t}}{2 (r - r_0)} \eta (1 + \| \xi \|_r^2) \right. \notag \\
		&\left.\quad+ \frac{e^{2 (r - r_0) t}}{2 r} \int_0^t E[\| H(X^{(\eta, \xi)}_s) - \mathcal{H}(X^{(\eta, \xi)}_s) \|_{\R^d}^2] ds \right\}. \label{EQ_Low_Order_Cause} 
	\end{align}
	Thus, $I_2(k \eta)$ is bounded by 
	\begin{align*}
		&C_1 e^{C_1 k \eta} \left[ \lambda \sqrt{\frac{C_1 \beta \varepsilon}{2 r_0}} \rho_{r, \delta}(\xi, \xi^\prime) 
		+ \left( \lambda \sqrt{\frac{C_1 \beta \varepsilon}{2 r_0}} + \frac{1}{\delta} \right) \right. \\
		&\left.\quad \times \left\{ \eta^{1/2} (1 + \| \xi \|_r) 
		+ \sup_{s \geq 0} E[\| H(X^{(\eta, \xi)}_s) - \mathcal{H}(X^{(\eta, \xi)}_s) \|_{\R^d}^2]^{1/2}
		\right\} \right]. 
	\end{align*}
	
	Finally, as in the proof of Lemma 2.5 in \cite{bao2020ergodicity}, 
	\begin{align*}
		I_3(t) 
		\leq \sqrt{3} \varepsilon^{-1/2} \delta \exp \left\{ \frac{3 \delta^2}{2 \varepsilon} \right\} \rho_{r, \delta}(\xi, \xi^\prime) 
	\end{align*}
	holds for $\delta < 1$ and $\xi, \xi^\prime \in \mathcal{C}_r$ with $\rho_{r, \delta}(\xi, \xi^\prime) < 1$.
	
	Putting these bounds all together, we can find some $C_2 = O_\alpha(1)$ so that $\mathbb{W}_{\rho_{r, \delta}}(\mathcal{P}_{k \eta}(\xi^\prime, \cdot), P^{(\eta)}_{k \eta}(\xi, \cdot))$ is bounded by
	\begin{align*}
		&C_2 \left\{ e^{- r_0 k \eta} + \sqrt{\beta \varepsilon} e^{C_2 k \eta} + \varepsilon^{-1/2} \delta e^{\frac{3}{2} \varepsilon^{-1} \delta} \right\} \rho_{r, \delta}(\xi, \xi^\prime) \\
		&\quad+ C_2 e^{C_2 k \eta} \left( \sqrt{\beta \varepsilon} + \frac{1}{\delta} \right) \left\{ \eta^{1/2} (1 + \| \xi \|_r) 
		+ \sup_{s \geq 0} E[\| H(X^{(\eta, \xi)}_s) - \mathcal{H}(X^{(\eta, \xi)}_s) \|_{\R^d}^2]^{1/2}
		\right\}
	\end{align*}
	holds for any $\varepsilon, \delta \in (0, 1)$ and $\xi, \xi^\prime \in \mathcal{C}_r$ with $\rho_{r, \delta}(\xi, \xi^\prime) < 1$. 
	Therefore, by taking $t = O_\alpha(1)$ to be sufficiently large so that $C_2 e^{- r_0 k \eta} \leq \theta / 2$ holds for any $t \leq k \eta \leq 2 t$ first, and then taking $\varepsilon = \Omega_\alpha(\beta^{-1})$ and $\delta = \Omega_\alpha(\beta^{-1/2})$ to be sufficiently small so that $C_2 (\sqrt{\beta \varepsilon} e^{2 C_2 t} + \varepsilon^{-1/2} \delta e^{\frac{3}{2} \varepsilon^{-1} \delta^2}) \leq \theta / 2$, we obtain the desired result. 
\end{proof}

\subsection{Time uniform bound on the difference between $\mathcal{P}_t$ and $P^{(\eta)}_t$}

\begin{thm}
	\label{Thm_Generalizaiton_of_Bao}
	Let $\alpha = (\| \nabla R(0) \|_{\R^d}, K, m, b, M, r, d)$ and let $\beta \geq 2 / m$ and $m / 3 > r$ hold.
	Furthermore, let $\tilde{\rho}_{r, \delta}(\xi, \xi^\prime) = \sqrt{\rho_{r, \delta}(\xi, \xi^\prime) (1 + \| \xi \|_r^4 + \| \xi^\prime \|_r^4)}$. 
	Then, there exist $c = \Omega_{\alpha, \beta}(1)$ and $C, t_0 = O_{\alpha, \beta}(1)$ such that $\mathbb{W}_{\tilde{\rho}_{r, \delta}}(P(\mathcal{X}^{(\xi^\prime)}_{k \eta} \in \cdot), P(X^{(\eta, \xi)}_{k \eta} \in \cdot))$ is bounded by 
	\begin{align}
		&\tilde{\rho}_{r, \delta}(\xi, \xi^\prime) e^{- c k \eta}
		+ (1 + \| \xi \|_r^2 + \| \xi^\prime \|_r^2) \notag \\
		&\quad\times C \left( \sup_{s \geq 0} E[\| H(X^{(\eta, \xi)}_s) - \mathcal{H}(X^{(\eta, \xi)}_s) \|_{\R^d}^2]^{1/4} 
		+ \eta^{1/4} (1 + \| \xi \|_r^{1/2}) \right) 
		\label{EQ_FSDE_Contraction}
	\end{align}
	uniformly on $0 < \eta \leq 1$ and $k$ satisfying $k \eta \geq t_0$. 
\end{thm}

\begin{proof}
	If we denote $C$ of Lemma \ref{Lem_Exi_Lyapunov_Function_Disc} by $C_1 = O_\alpha(1)$ and fix $\delta \in (0, 1)$, then by Lemmas \ref{Lem_Asymptotic_Boundedness} and \ref{Lem_d-smallness}, for any sufficiently large $t_* = O_{\alpha, \delta}(1)$,
	\begin{align*}
		\sup_{t_* \leq k \eta \leq 2 t_*} \mathbb{W}_{\rho_{r, \delta}}(P(\mathcal{X}^{(\xi^\prime)}_{k \eta} \in \cdot), P(X^{(\eta, \xi)}_{k \eta} \in \cdot)) 
		\leq 1 - \Omega_{\alpha, \delta}(1),\quad \xi, \xi^\prime \in \{ V_4 \leq 4 C_1 \}. 
	\end{align*}
	holds.  
	Furthermore, if $t_* = O_{\alpha, \beta}(1)$ is sufficiently large, then by Lemma \ref{Lem_Alpha_Contracting} with $r_0 = 2^{-1} \min \{ 1, r \}$, we can take $\delta = \Omega_{\alpha, \beta}(1)$ and $C_2 = O_{\alpha, \beta}(1)$ so that $\mathbb{W}_{\rho_{r, \delta}}(P(\mathcal{X}^{(\xi^\prime)}_{k \eta} \in \cdot), P(X^{(\eta, \xi)}_{k \eta} \in \cdot))$ is bounded by 
	\begin{align*}
		\frac{1}{2} \rho_{r, \delta}(\xi, \xi^\prime)
		+ C_2 \left\{ \sup_{s \geq 0} E[\| H(X^{(\eta, \xi)}_s) - \mathcal{H}(X^{(\eta, \xi)}_s) \|_{\R^d}^2]^{1/2} 
		+ \eta^{1/2} (1 + \| \xi \|_r) \right\}
	\end{align*}
	for any $k$ satisfying $t_* \leq k \eta \leq 2 t_*$. 
	As a result, for $k$ satisfying $t_* \leq k \eta \leq 2 t_*$, by (\ref{EQ_k_t_*_Ineq}) and Lemma \ref{Lem_Exi_Lyapunov_Function_Disc}, by retaking $t_* = O_{\alpha, \beta}(1)$ to be sufficiently large if necessary, we obtain (\ref{EQ_FSDE_Contraction}) for some $C = O_{\alpha, \beta}(1)$. 
	By repeating (\ref{EQ_FSDE_Contraction}), we can obtain (\ref{EQ_FSDE_Contraction}) for any $k \in \N$. 
\end{proof}

\section{Auxiliary Results on Coefficients of Adam-like Algorithms}
\label{SEC_Results_on_Coeff}

In this section, we evaluate the differences of functionals $H^{(\eta)}$ and $H_F$ of (\ref{EQ_Adam_Drift_Coeff_Disc}) and (\ref{EQ_Adam_Drift_Coeff_Conti}). 
As mentioned in Section \ref{SEC_Proof_Strategy}, to evaluate the differences between $X^{(\eta, \xi, L_n)}$ and $X^{(\xi, L_n)}$, or $X^{(\xi, L)}$ and $X^{(\xi, L_n)}$, we only have to evaluate the differences between $H^{(\eta)}_{L_n}$ and $H_{L_n}$, or $H_L$ and $H_{L_n}$, respectively. 
Thus, the following results enable us to prove Propositions \ref{Lem_StepSize_and_Gen_Bound_Adam} and \ref{Lem_Gradient_Vanish_in_Adam} in Appendix \ref{SEC_Proof_Main}.  

\begin{lem}
	\label{Lem_Lipschitz_Continuity_Conti}
	Suppose that $c_1, c_2 > 0$ and $r > 0$ satisfy $2 c_1 > c_2$, $\min \{ c_1, c_2 \} > r$, and $F : \R^d \to \R$ is $M$-smooth and satisfies $\| \nabla F \|_\infty \coloneqq \sup_{x \in \R^d} \| \nabla F(x) \|_{\R^d} < \infty$. 
	Then, there exists some $C_1 = O_{c_1, c_2, d}(1)$ and $C_2 = O_{c_1, c_2, \varepsilon, \| \nabla F \|_\infty, M, r, d}(1)$ such that the functional $H^{(\eta)}_F : \mathcal{C}_r \to \R^d$ defined by (\ref{EQ_Adam_Drift_Coeff_Conti}) is bounded by $C_1$ and Lipschitz continuous with Lipschitz constant $C_2$. 
	The same results hold for the functional $H_F : \mathcal{C}_r \to \R^d$ uniformly on $\eta > 0$. 
\end{lem}

\begin{proof}
	We only prove the results for $H^{(\eta)}_F : \mathcal{C}_r \to \R^d$. 
	Furthermore, considering each coordinate, we may assume $d = 1$. 
	For the numerator of
	\begin{align*}
		H^{(\eta)}_F(\xi) 
		= - \displaystyle{\frac{(1 - e^{- c_1 \eta}) \sum_{j=-\infty}^0 e^{c_1 j \eta} \nabla F(\xi(j \eta))}{\sqrt{ \varepsilon + (1 - e^{- c_2 \eta}) \sum_{j=-\infty}^0 e^{c_2 j \eta} |\nabla F(\xi(j \eta))|^2}}}, 
	\end{align*} 
	we have 
	\begin{align*}
		&\left| (1 - e^{- c_1 \eta}) \sum_{j=-\infty}^0 e^{c_1 j \eta} \nabla F(\xi(j \eta)) \right| \\
		&\quad\leq (1 - e^{- c_1 \eta}) \left( \sum_{j=-\infty}^0 e^{(2 c_1 - c_2) j \eta} \right)^{1/2} \left( \sum_{j=-\infty}^0 e^{c_2 j \eta} |\nabla F(\xi(j \eta))|^2 \right)^{1/2} \\
		&\quad= \frac{1 - e^{- c_1 \eta}}{\sqrt{1 - e^{- (2 c_1 - c_2) \eta}}} \left( \sum_{j=-\infty}^0 e^{(2 c_1 - c_2) j \eta} \right)^{1/2} \left( \sum_{j=-\infty}^0 e^{c_2 j \eta} |\nabla F(\xi(j \eta))|^2 \right)^{1/2}. 
	\end{align*}
	Therefore, the functional $H^{(\eta)}_F$ is bounded by $C_1 = \sup_{\eta > 0} \frac{1 - e^{- c_1 \eta}}{\sqrt{(1 - e^{- (2 c_1 - c_2) \eta}) (1 - e^{- c_2 \eta})}}$. 
	
	Let $\xi, \xi^\prime \in \mathcal{C}_r$ satisfy $I(\xi^\prime) \leq I(\xi)$, where $I(\xi) = (1 - e^{- c_2 \eta}) \sum_{j=-\infty}^0 e^{c_2 j \eta} |\nabla F(\xi(j \eta))|^2$. 
	Then, we have
	\begin{align*}
		|H^{(\eta)}_F(\xi) - H^{(\eta)}_F(\xi^\prime)| 
		&\leq \frac{1 - e^{- c_1 \eta}}{\sqrt{\varepsilon + I(\xi^\prime)}} \sum_{j=-\infty}^0 e^{c_1 j \eta} | \nabla F(\xi(j \eta)) - \nabla F(\xi^\prime(j \eta)) | \\
		&\quad+ (1 - e^{- c_1 \eta}) \left( \frac{1}{\sqrt{\varepsilon + I(\xi)}} - \frac{1}{\sqrt{\varepsilon + I(\xi^\prime)}} \right)  \left| \sum_{j=-\infty}^0 e^{c_1 j \eta} \nabla F(\xi(j \eta)) \right|. 
	\end{align*}
	By the $M$-smoothness of $F$, the first term in R.H.S. is bounded by
	\begin{align*}
		\frac{(1 - e^{- c_1 \eta}) M}{\sqrt{\varepsilon}} \sum_{j=-\infty}^0 e^{c_1 j \eta} | \xi(j \eta) - \xi^\prime(j \eta) | 
		&\leq \frac{(1 - e^{- c_1 \eta}) M}{\sqrt{\varepsilon}} \sum_{j=-\infty}^0 e^{(c_1 - r) j \eta} \| \xi - \xi^\prime \|_r \\
		&= \frac{(1 - e^{- c_1 \eta}) M}{(1 - e^{- (c_1 - r) \eta}) \sqrt{\varepsilon}} \| \xi - \xi^\prime \|_r. 
	\end{align*}
	Since $a^{-1/2} - b^{-1/2} \leq 2 ^{-1} a^{-3/2} (b - a)$ holds by the Taylor's theorem for any $0 < a \leq b$, the second term in R.H.S. is bounded by
	\begin{align*}
		\frac{|H^{(\eta)}_F(\xi)|}{2 \sqrt{\varepsilon + I(\xi)}} |I(\xi) - I(\xi^\prime)| 
		&\leq  \frac{C_1 M \| \nabla F \|_\infty (1 - e^{- c_2 \eta})}{\sqrt{\varepsilon + I(\xi)}} \sum_{j=-\infty}^0 e^{(c_2 - r) j \eta} \| \xi - \xi^\prime \|_r \\
		&\leq \frac{C_1 M \| \nabla F \|_\infty (1 - e^{- c_2 \eta})}{\sqrt{\varepsilon} (1 - e^{- (c_2 - r) \eta})} \| \xi - \xi^\prime \|_r.  
	\end{align*}
	Therefore, $H^{(\eta)}_F : \mathcal{C}_r \to \R^d$ is Lipschitz continuous with Lipschitz constant 
	\begin{align*}
		C_2 
		= \frac{M}{\sqrt{\varepsilon}} \sup_{\eta > 0} \left( \frac{1 - e^{- c_1 \eta}}{1 - e^{- (c_1 - r) \eta}} + \frac{C_1 \| \nabla F \|_\infty (1 - e^{- c_2 \eta})}{1 - e^{- (c_2 - r) \eta}} \right). 
	\end{align*} 
	
\end{proof}

\begin{lem}
	\label{Lem_Drift_Difference_Disc_Conti}
	Under the same conditions in Lemma \ref{Lem_Lipschitz_Continuity_Conti}, the following inequality holds uniformly on $0 < \eta \leq 1$ and $\xi \in \mathcal{C}_r$. 
	\begin{align*}
		&\| H_F(\xi) - H^{(\eta)}_F(\xi) \|_{\R^d} \\
		&\quad\leq O_{c_1, c_2, \varepsilon, \| \nabla F \|_\infty, M, r, d}\left(\eta + \int_{-\infty}^0 e^{(c_1 \wedge c_2) s} \|\nabla F(\xi(s)) - \nabla F(\xi(\lfloor s / \eta \rfloor \eta))\|_{\R^d} ds\right). 
	\end{align*}
\end{lem}

\begin{proof}
	As in the proof of Lemma \ref{Lem_Lipschitz_Continuity_Conti} we may assume $d = 1$. 
	Then, we have
	\begin{align*}
		|H_F(\xi) - H_F^{(\eta)}(\xi)| 
		&\leq \frac{1}{\sqrt{\varepsilon + \tilde{I}(\xi)}} \left| c_1 \int_{-\infty}^0 e^{c_1 s} \nabla F(\xi(s)) ds - (1 - e^{- c_1 \eta}) \sum_{j=-\infty}^0 e^{c_1 j \eta} \nabla F(\xi(j \eta)) \right| \\
		&\quad+ (1 - e^{- c_1 \eta}) \left| \frac{1}{\sqrt{\varepsilon + \tilde{I}(\xi)}} - \frac{1}{\sqrt{\varepsilon + I(\xi)}} \right|  \left| \sum_{j=-\infty}^0 e^{c_1 j \eta} \nabla F(\xi(j \eta)) \right|,  
	\end{align*}
	where $I(\xi) = (1 - e^{- c_2 \eta}) \sum_{j=-\infty}^0 e^{c_2 j \eta} |\nabla F(\xi(j \eta))|^2$ and $\tilde{I}(\xi) = c_2 \int_{-\infty}^0 e^{c_2 s} |\nabla F(\xi(s))|^2 ds$. 
	Since $|(1 - e^{- c_1 \eta}) - c_1 \eta| \leq \frac{c_1^2 \eta^2}{2}$ holds and $e^{c_1 j \eta} - e^{c_1 s} \leq c_1 \eta$ when $(j+1) \eta \leq s < j \eta$, the first term in R.H.S. is bounded by
	\begin{align*}
		&\frac{\eta^2 c_1}{\sqrt{\varepsilon}} \left( \frac{c_1}{2} + 1 \right) \left| \sum_{j=-\infty}^0 e^{c_1 j \eta} \nabla F(\xi(j \eta)) \right| 
		+ \frac{c_1}{\sqrt{\varepsilon}} \int_{-\infty}^0 e^{c_1 s} |\nabla F(\xi(s)) - \nabla F(\xi(\lfloor s / \eta \rfloor \eta))| ds \\
		&\quad\leq \frac{\eta^2 c_1 \| \nabla F \|_\infty}{(1 - e^{- c_1 \eta})\sqrt{\varepsilon}} \left( \frac{c_1}{2} + 1 \right) 
		+ \frac{c_1}{\sqrt{\varepsilon}} \int_{-\infty}^0 e^{c_1 s} |\nabla F(\xi(s)) - \nabla F(\xi(\lfloor s / \eta \rfloor \eta))| ds.
	\end{align*}
	Similarly, the second term in R.H.S. is bounded by
	\begin{align*}
		&\frac{\| \nabla F \|_\infty}{2 \sqrt{\varepsilon^3}} \left| \tilde{I}(\xi) - I(\xi) \right| \\
		&\quad\leq \frac{\eta^2 c_2 \| \nabla F \|_\infty^3}{(1 - e^{- c_2 \eta}) \sqrt{\varepsilon^3}} \left( \frac{c_2}{2} + 1 \right) + \frac{c_2 \| \nabla F \|_\infty}{\sqrt{\varepsilon^3}} \int_{-\infty}^0 e^{c_2 s} |\nabla F(\xi(s)) - \nabla F(\xi(\lfloor s / \eta \rfloor \eta))| ds. 
	\end{align*}
\end{proof}

\begin{lem}
	\label{Lem_Drift_Difference_Empirical_True} 
	Let the the same conditions in Lemma \ref{Lem_Lipschitz_Continuity_Conti} hold and let $G : \R^d \to \R$ be $M$-smooth with $\| \nabla G \|_\infty < \infty$. 
	Then, the following inequality holds uniformly on $\xi \in \mathcal{C}_r$. 
	\begin{align*}
		\|H_F(\xi) - H_G(\xi)\|_{\R^d} 
		\leq O_{c_1, c_2, \varepsilon, \| \nabla F \|_\infty, \| \nabla G \|_\infty, M, r, d}\left( \int_{-\infty}^0 e^{(c_1 \wedge c_2) s} \|\nabla F(\xi(s)) - \nabla G(\xi(s))\|_{\R^d} ds\right). 
	\end{align*}
\end{lem}

\begin{proof}
	We only have to consider the case of $d = 1$. 
	Then, we have 
	\begin{align*}
		|H_F(\xi) - H_G(\xi)| 
		&\leq \frac{c_1}{\sqrt{\varepsilon + I_F(\xi)}} \int_{-\infty}^0 e^{c_1 s} |\nabla F(\xi(s)) - \nabla G(\xi(s))| ds \\
		&\quad+ \| \nabla G \|_\infty \left| \frac{1}{\sqrt{\varepsilon + I_F(\xi)}} - \frac{1}{\sqrt{\varepsilon + I_G(\xi)}} \right| 
	\end{align*}
	where $I_F(\xi) = c_2 \int_{-\infty}^0 e^{c_2 s} |\nabla F(\xi(s))|^2 ds$ and $I_G(\xi) = c_2 \int_{-\infty}^0 e^{c_2 s} |\nabla G(\xi(s))|^2 ds$. 
	The rest of the proof is quite similar to those of Lemmas \ref{Lem_Lipschitz_Continuity_Conti} and \ref{Lem_Drift_Difference_Disc_Conti}, and is omitted. 
\end{proof}

\begin{lem}
	\label{Lem_Drift_Difference_SGD_Adam}
	under the same conditions in Lemma \ref{Lem_Lipschitz_Continuity_Conti}, for $\alpha = (c_1, c_2, \varepsilon, \| \nabla F \|_\infty, M, d)$, the solution $X^{(\xi, F)}$ of (\ref{EQ_Regularized_Adam_Conti}) satisfies
	\begin{align*}
		&E\left[ \left\| H_F(X^{(\xi, F)}_t) - (\varepsilon + \| \nabla F(X^{(\xi, F)}(t/2)) \|_{\R^d}^2)^{-1/2} \nabla F(X^{(\xi, F)}(t)) \right\|_{\R^d}^2 \right] \\
		&\quad\leq O_\alpha \left( E[\varDelta_1(X^{(\xi, F)}_t)] + E[\varDelta_2(X^{(\xi, F)}_t)] \right)
	\end{align*}
	uniformly on $t > 0$ and $\xi \in \mathcal{C}_r$. 
	Here, $\varDelta_1$ and $\varDelta_2$ are defined by
	\begin{align*}
		\varDelta_1(X^{(\xi, F)}_t) 
		&= \frac{2 c_1}{\sqrt{\varepsilon}} \int_{-\infty}^0 e^{2 c_1 s} \|\nabla F(X^{(\xi, F)}_t(s)) - \nabla F(X^{(\xi, F)}(t)) \|_{\R^d} ds, \\ 
		\varDelta_2(X^{(\xi, F)}_t) 
		&= \frac{c_2}{2 \sqrt{\varepsilon^3}} \left\{ \int_{-\infty}^0 e^{c_2 s} \left| \| \nabla F(X^{(\xi, F)}_t(s)) \|_{\R^d}^2 - \| \nabla F(X^{(\xi, F)}(t/2)) \|_{\R^d}^2 \right| ds \right\},  
	\end{align*}
	respectively. 
\end{lem}

\begin{proof}
	Simple consequence of the Taylor's theorem as in the proof of Lemma \ref{Lem_Lipschitz_Continuity_Conti}. 
\end{proof}

\begin{lem}
	\label{Lem_Adam_Grad_Vanish}
	Suppose that $\nabla F(x) = 0$ whenever $\nabla R(x) \neq 0$. 
	Then, under the same condition as Lemma \ref{Lem_Drift_Difference_SGD_Adam}, we have
	\begin{align*}
		E[\| \nabla F(X^{(\xi, F)}(t)) \|_{\R^d}^2] 
		&\leq O_\alpha \left( \frac{1}{\beta} + \int_{-\infty}^0 e^{2 c_1 s} \left\{ E[|F(X^{(\xi, F)}_t(s)) - F(X^{(\xi, F)}(t))| \right.\right. \\
		&\left.\left. \quad+ \| \nabla F(X^{(\xi, F)}_t(s)) - \nabla F(X^{(\xi, F)}(t)) \|_{\R^d}^2] \right\} ds  \right).
	\end{align*}
\end{lem}

\begin{proof}
	By the assumption and Ito's rule, we have 
	\begin{align}
		d F(X^{(\xi, F)}(t)) 
		&\leq - \langle \nabla F(X^{(\xi, F)}(t)), H_F(X_t^{(\xi, F)}) \rangle_{\R^d} dt \notag \\
		&\quad+ \sqrt{\frac{2}{\beta}} \langle \nabla F(X^{(\xi, F)}(t)), d W(t) \rangle_{\R^d} 
		+ \frac{Md}{\beta} dt. 
		\label{EQ_Ito_for_Loss}
	\end{align}
	Therefore, 
	\begin{align*}
		d [e^{2 c_1 t} F(X^{(\xi, F)}(t))] 
		&= 2 c_1 e^{2 c_1 t} F(X^{(\xi, F)}(t)) dt + e^{2 c_1 t} d F(X^{(\xi, F)}(t)) \\
		&\leq 2 c_1 e^{2 c_1 t} F(X^{(\xi, F)}(t)) dt - e^{2 c_1 t} \langle \nabla F(X^{(\xi, F)}(t)), H_F(X_t^{(\xi, F)}) \rangle_{\R^d} dt \\
		&\quad+ \sqrt{\frac{2}{\beta}} e^{2 c_1 t} \langle \nabla F(X^{(\xi, F)}(t)), d W(t) \rangle_{\R^d} 
		+ \frac{Md}{\beta} e^{2 c_1 t} dt
	\end{align*}
	holds. 
	Since  
	\begin{align*}
		\frac{d}{dt} \left[ \left\| e^{c_1 t} \int_{-\infty}^0 e^{c_1 s} \nabla F(\xi^\prime_t(s)) ds \right\|_{\R^d}^2 \right] 
		&= \frac{d}{dt} \left[ \left\| \int_{-\infty}^t e^{c_1 s} \nabla F(\xi^\prime(s)) ds \right\|_{\R^d}^2 \right] \\
		&= 2 e^{c_1 t} \langle \nabla F(\xi^\prime(t)) \int_{-\infty}^t e^{c_1 s} \nabla F(\xi^\prime(s)) ds \rangle_{\R^d} \\
		&= 2 e^{2 c_1 t} \langle \nabla F(\xi^\prime(t)) \int_{-\infty}^0 e^{c_1 s} \nabla F(\xi^\prime_t(s)) ds \rangle_{\R^d} 
	\end{align*}
	holds for any $\xi^\prime \in \mathcal{C}_r$, this indicates 
	\begin{align*}
		&\frac{1}{\sqrt{\varepsilon + \| \nabla F \|_{\infty}^2}} d \left[ \left\| e^{c_1 t} \int_{-\infty}^0 e^{c_1 s} \nabla F(X_t^{(\xi, F)}(s)) ds \right\|_{\R^d}^2 \right] \\
		&\quad\leq 2 c_1 e^{2 c_1 t} F(X^{(\xi, F)}(t)) dt - d [e^{2 c_1 t} F(X^{(\xi, F)}(t))] \\
		&\qquad+ \sqrt{\frac{2}{\beta}} e^{2 c_1 t} \langle \nabla F(X^{(\xi, F)}(t)), d W(t) \rangle_{\R^d} 
		+ \frac{Md}{\beta} e^{2 c_1 t} dt. 
	\end{align*}
	Multiplying both sides by $e^{- 2 c_1 t}$ after integrating over $(-\infty, t]$, we obtain 
	\begin{align*}
		&\frac{1}{\sqrt{\varepsilon + \| \nabla F \|_{\infty}^2}} E\left[ \left\| \int_{-\infty}^0 e^{c_1 s} \nabla F(X_t^{(\xi, F)}(s)) ds \right\|_{\R^d}^2 \right] \\
		&\quad\leq 2 c_1 \int_{-\infty}^0 e^{2 c_1 s} E[F(X^{(\xi, F)}_t(s))] ds - E[F(X^{(\xi, F)}(t))] + \frac{M d}{2 c_1 \beta}. 
	\end{align*}
	Therefore, by
	\begin{align*}
		E[\| \nabla F(X^{(\xi, F)}(t)) \|_{\R^d}^2] 
		&\leq 2 E\left[ \left\| \int_{-\infty}^0 e^{c_1 s} \nabla F(X_t^{(\xi, F)}(s)) ds - \nabla F(X^{(\xi, F)}(t)) \right\|_{\R^d}^2 \right] \\
		&+ 2 E\left[ \left\| \int_{-\infty}^0 e^{c_1 s} \nabla F(X_t^{(\xi, F)}(s)) ds \right\|_{\R^d}^2 \right], 
	\end{align*}
	the proof is completed. 
\end{proof}

\section{Proofs of Results in Main Section}
\label{SEC_Proof_Main}

\subsection{Proof of Theorem \ref{Thm_General_Bound_on_Difference}}

Let $\delta \in (0, 1)$ be fixed. 
Since $\| \nabla F(x) \|_{\R^d} \leq M \| x \|_{\R^d} + \| \nabla F(0) \|_{\R^d}$ holds for any $x \in \R^d$, as in the proof of Lemma 6 in \cite{Ragi}, we have
\begin{align*}
	F(x) - F(y) 
	\leq \left( \frac{M}{2} \| x \|_{\R^d} + \frac{M}{2} \| y \|_{\R^d} + \| \nabla F(0) \|_{\R^d} \right) \| x - y \|_{\R^d},\quad x, y \in \R^d. 
\end{align*}
Therefore, combining this inequality with 
\begin{align*}
	\| x - y \|_{\R^d} 
	\leq 6 (1 + \sqrt{\delta}) \sqrt{\{ 1 \wedge (\delta^{-1} \| x - y \|_{\R^d}) \} (1 + \| x \|_{\R^d}^2 + \| y \|_{\R^d}^2)},\quad x, y \in \R^d, 
\end{align*}
we can take a constant $C = O_{\| \nabla F(0) \|_{\R^d}, M}(1)$ so that 
\begin{align*}
	F(x) - F(y) 
	\leq C \sqrt{\{ 1 \wedge (\delta^{-1} \| x - y \|_{\R^d}) \} (1 + \| x \|_{\R^d}^4 + \| y \|_{\R^d}^4)}
\end{align*}
holds. 
As a result, for any probability measures $\mu, \nu$ on $\R^d$ and $\pi \in \Pi(\mu, \nu)$, we obtain 
\begin{align}
	&\int_{\R^d} H(x) \mu(dx) - \int_{\R^d} H(x) \nu(dx) \notag \\
	&\quad\leq C \int_{\R^d \times \R^d} \sqrt{\{ 1 \wedge (\delta^{-1} \| x - y \|_{\R^d}) \} (1 + \| x \|_{\R^d}^4 + \| y \|_{\R^d}^4)} \pi(dx dy). \label{EQ_Coupling_Inequality_Expectation}
\end{align}
Putting $\mu = P(\mathcal{X}^{(\xi^\prime)}(k \eta) \in \cdot)$ and $\nu = P(X^{(\eta, \xi)}(k \eta) \in \cdot)$, and using Theorem \ref{Thm_Generalizaiton_of_Bao}, we obtain the desired result. 
\qed

\subsection{Proof of the first inequality in Proposition \ref{Lem_StepSize_and_Gen_Bound_Adam}}

According to Lemma \ref{Lem_Drift_Difference_Disc_Conti}, 
\begin{align*}
	&E[\| H_{L_n}(X^{(\eta, \xi, L_n)}_t) - H^{(\eta)}_{L_n}(X^{(\eta, \xi, L_n)}_t) \|_{\R^d}^2] \\
	&\quad\leq O_\alpha \left(\eta^2 + \int_{-\infty}^0 e^{(c_1 \wedge c_2) s} E[\|\nabla L_n(X^{(\eta, \xi, L_n)}_t(s)) - \nabla L_n(X^{(\eta, \xi, L_n)}_t(\lfloor s / \eta \rfloor \eta))\|_{\R^d}^2] ds\right) 
\end{align*}
holds. 
Therefore, since Lemma \ref{Lem_Exi_Lyapunov_Function_Disc} indicates 
\begin{align*}
	E[\|\nabla L_n(X^{(\eta, \xi, L_n)}_t(s)) - \nabla L_n(X^{(\eta, \xi, L_n)}_t(\lfloor s / \eta \rfloor \eta))\|_{\R^d}^2] 
	\leq O_\alpha( (1 + \| \xi \|_r^2) \eta), 
\end{align*}
by Theorem \ref{Thm_General_Bound_on_Difference}, we obtain the desired result. 
\qed

\subsection{Proof of the second inequality in Proposition \ref{Lem_StepSize_and_Gen_Bound_Adam}}

Let $S = (z_1, \dots, z_n)$ and $S^\prime = (z_1^\prime, \dots, z_n^\prime)$ be neighboring datasets, that is, there exists at most one $i$ satisfying $z_i \neq z_i^\prime$. 
Furthermore, let $L_n^\prime(w) = \frac{1}{n} \sum_{i=1}^n \ell(x; z_i^\prime)$. 
Then, by Lemma \ref{Lem_Drift_Difference_Empirical_True}, we have 
\begin{align*}
	&E[\| H_{L_n}(X^{(\xi, L_n)}_t) - H_{L_n^\prime}(X^{(\xi, L_n)}_t) \|_{\R^d}^2] \\
	&\quad\leq O_\alpha\left( \int_{-\infty}^0 e^{(c_1 \wedge c_2) s} E[\|\nabla L_n(X^{(\xi, L_n)}_t(s)) - \nabla L_n^\prime(X^{(\xi, L_n)}_t(s))\|_{\R^d}^2] ds\right) \\
	&\quad\leq O_\alpha \left( \frac{1 + \| \xi \|_r^2}{n^2} \right) 
\end{align*}
uniformly on $t \geq 0$. 
Here, we used Lemma \ref{Lem_Exi_Lyapunov_Function_Disc} in the second inequality. 
By Lemma \ref{Lem_Lipschitz_Continuity_Conti}, we can apply Theorem \ref{Thm_General_Bound_on_Difference} for $X^{(\xi, L_n)}$ and $X^{(\xi, L_n^\prime)}$. 
Therefore, by (\ref{EQ_General_Bound_on_Difference_Conti}), we obtain 
\begin{align*}
	\left| E[(\ell(\cdot; z) + \varepsilon^{1/2} R)(X^{(\xi, L_n)}(t))] - E[(\ell(\cdot; z) + \varepsilon^{1/2} R)(X^{(\xi, L_n^\prime)}(t))] \right| 
	&\leq  O_{\alpha, \beta} \left( \frac{1 + \| \xi \|_r^{5/2}}{\sqrt{n}} \right) 
\end{align*}
for any $z \in \mathcal{Z}$. 
The desired result follows from Theorem 2.2 in \cite{Hadt}. 
\qed

\subsection{Proof of Proposition \ref{Lem_Gradient_Vanish_in_Adam}}

Let $Y(s, \cdot)$ be the solution of 
\begin{align}
	\label{EQ_SGLD}
	\begin{cases}
		dY(s, t)
		= - c(\varepsilon)^{-1} \nabla L_n(Y(s, t)) dt - \nabla R(Y(s, t)) + \sqrt{2 / \beta} dW(t), & t > s \\
		Y(s, t) = X^{(\xi, L_n)}(s), & s \geq t,  
	\end{cases}
\end{align}
where $c(\varepsilon) = (\varepsilon + \| \nabla L_n(X^{(\xi, L_n)}(s)) \|_{\R^d}^2)^{1/2}$. 
According to Theorem \ref{Thm_General_Bound_on_Difference}, for the invariant measure 
\begin{align}
	\label{EQ_Inv_Meas_SGD}
	\pi^*(dw) 
	\propto \exp \left\{ - \frac{\beta}{c(\varepsilon)} L_n(w) - \beta R(w) \right\} dw
\end{align}
of $P(Y(s, t) \in \cdot)$, we can take constants $c = \Omega_{\alpha, \beta}(1)$ and $C = O_{\alpha, \beta}(1)$ so that the exponential convergence
\begin{align*}
	\left| E[L_n(Y(t/2, t)) +  c(\varepsilon) R(Y(t/2, t))] - \int_{\R^d} \{ L_n(x) + c(\varepsilon) R(x) \} \pi^*(dx) \right| 
	&\leq C e^{- c t} (1 + \| \xi \|_r^2) 
\end{align*}
holds. 
Similarly, there exists the unique invariant measure $\mu_*$ of $P(X^{(\xi, L_n)}_t \in \cdot)$ by Proposition \ref{Thm_Existence_Invariant_Measure_Origin}, and it satisfies the exponential convergence
\begin{align*}
	\left| E[H(X^{(\xi, L_n)}_t)] - \int_{\mathcal{C}_r} H(\xi^\prime) \mu_*(d \xi^\prime) \right| 
	\leq O_K (C e^{- c t} (1 + \| \xi \|_r^2))
\end{align*}
for any bounded and Lipschitz functional $H : \mathcal{C}_r \to \R$ with Lipschitz constant $K$. 
Combining this inequality with Lemma \ref{Lem_Drift_Difference_SGD_Adam}, we obtain
\begin{align*}
	E\left[ \left\| H_L(X^{(\xi, L_n)}(t)) - c(\varepsilon)^{-1} \nabla L_n(X^{(\xi, L_n)}(t)) \right\|_{\R^d}^2 \right]
	\leq O_\alpha \left( e^{- c t} (1 + \| \xi \|_r^2) \right). 
\end{align*}
Therefore, by Theorem \ref{Thm_General_Bound_on_Difference}, 
\begin{align*}
	\left| E[(L_n + c(\varepsilon) R)(X^{(\xi, L_n)}(t))] - E[(L_n + c(\varepsilon) R)(Y(t/2, t))] \right| 
	&\leq O_\alpha\left( e^{- c t / 4} (1 + \| \xi \|_r^{5/2}) \right) 
\end{align*}
holds. 

On the other hand, by (3.4) in \cite{Ragi}, we have
\begin{align*}
	&\int_{\R^d} (L_n + c(\varepsilon) R)(x) \pi^*(dx) 
	- \min_{w \in \R^d} (L_n + c(\varepsilon) R)(w) \\
	&\quad= c(\varepsilon) \left( \int_{\R^d} (c(\varepsilon) ^{-1}L_n + R) \pi^*(dx)  
	- \min_{w \in \R^d} (c(\varepsilon) ^{-1}L_n + R)(w) \right) \\
	&\quad\leq c(\varepsilon) \frac{\log (\beta + 1)}{\beta}. 
\end{align*}
Combining this with the inequality $E[\min_{w \in \R^d} (L_n + c(\varepsilon) R)(w)] \leq \min_{w \in \R^d} (L_n + c(\varepsilon) R)(w)$, we obtain 
\begin{align*}
	E[\int_{\R^d} (L_n + c(\varepsilon) R)(x) \pi^*(dx)] 
	- E[\min_{w \in \R^d} (L + c(\varepsilon) R)(w)] 
	\leq E[c(\varepsilon)] \frac{\log (\beta + 1)}{\beta}.
\end{align*}
Finally, for $w^* = \argmin_{w \in \R^d} (L + \varepsilon^{1/2} R)(w)$,  
\begin{align*}
	E[\min_{w \in \R^d} (L + c(\varepsilon) R)(w)] 
	- \min_{w \in \R^d} (L + \varepsilon^{1/2} R)(w) 
	&\leq \frac{R(w^*)}{2 \sqrt{\varepsilon}} E[\| \nabla L_n(X^{(\xi, F)}(t/2)) \|_{\R^d}^2]
\end{align*}
holds. 
Since we have $E[\| \nabla L_n(X^{(\xi, F)}(t/2)) \|_{\R^d}^2] \leq O_\alpha(\beta^{-1} + e^{- c t} (1 + \| \xi \|_r^2))$ by Lemma \ref{Lem_Adam_Grad_Vanish}, the proof is completed.
\qed 


\bibliographystyle{plain}
\bibliography{article.bib}

\clearpage

\end{document}